\documentclass{article} % For LaTeX2e
\usepackage{iclr2027_conference,times}
\usepackage{newtxtext}
\usepackage{newtxmath}

\usepackage{amsmath}
\usepackage{amsthm}
\usepackage{hyperref}
\usepackage{url}
\usepackage{booktabs}
\usepackage{graphicx} 
\newtheorem{theorem}{Theorem}
\newtheorem*{theorem*}{Theorem}
\newtheorem{assumption}{Assumption}
\newtheorem{proposition}{Proposition}
\newtheorem{example}{Example}
\newtheorem{remark}{Remark}
\usepackage{multirow}
\usepackage{subcaption}
\usepackage{wrapfig}
\usepackage{placeins}

\newcommand{\modelname}{SplitJEPA }
\usepackage{amsmath,amsfonts,bm}

\def\eqref#1{equation~\ref{#1}}
\def\1{\bm{1}}

\DeclareMathAlphabet{\mathsfit}{\encodingdefault}{\sfdefault}{m}{sl}
\SetMathAlphabet{\mathsfit}{bold}{\encodingdefault}{\sfdefault}{bx}{n}

\title{SplitJEPA: Learning Invariant and Variant Latent Worlds without Reconstruction}
\author{
Ruijin Hua$^1$,
Zichuan Liu$^2$,
Zhuokai Zhao$^3$,
Yujia Zheng$^{1}$
\\[6pt]
$^1$ University of Illinois Urbana-Champaign
\hfill
$^2$ Carnegie Mellon University
\hfill
$^3$ University of Chicago
}

\iclrfinalcopy 
\begin{document}

\maketitle

\begingroup
\renewcommand{\thefootnote}{}
\footnotetext{Project website: \url{https://splitjepa.github.io/}}
\addtocounter{footnote}{-1}
\endgroup

\vspace{-10pt}
\begin{abstract}

Understanding a dynamical world calls for more than a latent state that summarizes its observations: the state should also be organized into the factors that stay shared across related observations and the factors that vary between them. For example, a robot pushing a cube to a goal should take the same action when the camera shifts or the lights dim, since nothing in the scene has moved. Existing approaches to this decomposition commonly obtain it through reconstruction, so the latent variables must first explain the entire observational world before their organization can be trusted. Joint embedding predictive architectures (JEPAs) model the latent state directly and never reconstruct, yet no existing result recovers the invariant and variant parts of the state they learn. How to learn the invariant-variant structure of the latent world without paying for its reconstruction therefore remains open. To close this gap, we introduce \textbf{SplitJEPA}, a JEPA that jointly recovers the latent state and its invariant and variant organization directly in representation space, without any reconstruction. We prove that, under stationary Gaussian predictive dynamics and a full-rank variation condition, \modelname identifies the invariant and variant subspaces up to independent block-wise isometries, without introducing an observation decoder. Since the guarantee needs no decoder, the result extends reconstruction-free latent recovery to invariant-variant block identification.  Experiments on synthetic nonlinear systems and robotic manipulation tasks support the theoretical results and show their practical value for both robustness and efficiency.

\end{abstract}
\vspace{-10pt}

\section{Introduction}
Learning a useful representation of a dynamical world requires more than preserving information about its observations. 
The representation should also reveal how the latent state is structured into factors that remain stable and factors that vary across relevant changes~\citep{scholkopf2021causalrepresentationlearning,kugelgen2021selfsupervised}. Robot manipulation provides an example. Different executions of the same task may vary in robot configuration, local motion, viewpoint, illumination, and object appearance. A useful representation should retain the invariant content while separating the sources of variation, allowing downstream computations to use stable information without discarding variation when it is useful. This motivates structured representation learning: recovering not only the global latent state, but also its invariant–variant decomposition.

This objective has inspired invariant representation learning from multiview and temporal robotic data~\citep{sermanet2017time,nair2022rm}, as well as content-style approaches that seek to separate shared information from varying factors~\citep{shrestha2025contentstyle,timilsina2026contentstyle,wang2025an}. These approaches commonly rely on reconstruction or related generative objectives: an encoder maps observations into latent variables, while the decoder requires those variables to keep enough information to reproduce the observations. Although effective, reconstruction introduces additional modeling and computational costs and can direct representation capacity toward noise and observation details that are unnecessary for downstream tasks. More fundamentally, reconstruction encourages the representation to preserve information needed to reproduce the observations, but it does not by itself specify how that information should be partitioned into invariant and variant structure. Indeed, an invertible nonlinear reparameterization may preserve the complete observation distribution while mixing the underlying latent factors~\citep{hyvarinen1999nonlinear,locatello2019challengingcommonassumptionsunsupervised}. Identifiable generative models address this ambiguity by introducing additional supervision or functional constraints~\citep{hyvarinen2024identifiability,moran2026toward}, illustrating that these additional structures determine the latent organization. Our goal is to explore whether suitable structure can identify the desired decomposition directly in representation space without requiring an observation decoder. This raises the core question of our work: \emph{can the invariant--variant structure of a dynamical latent state be identified without reconstructing observations?}

Reconstruction-free learning offers a promising foundation. Joint-Embedding Predictive Architectures (JEPAs) learn informative representations by making predictions directly in representation space without reconstruction~\citep{assran2023selfsupervisedlearningimagesjointembedding,bardes2024revisiting,Dawid_2024}. Related robotic studies further show that useful invariances can be learned without pixel reconstruction and that separating view-invariant from view-dependent information can improve visual world models~\citep{zhang2021learning,pang2025learning}.  However, these empirical results do not guarantee recovery of the ground-truth latent structure or its invariant-variant decomposition. Contrastive objectives can reduce nonlinear latent ambiguities by exploiting relationships between observations under asymmetry assumptions~\citep{hyvarinen2016unsupervised,hyvarinen2017nonlinear,zimmermann2022contrastivelearninginvertsdata}. At the same time, block identification does not significantly require identifying individual coordinates within each block. LeJEPA provides a complementary result in this direction. Under stationary Gaussian predictive dynamics and appropriate representation normalization, it identifies the complete latent process~\citep{balestriero2025lejepaprovablescalableselfsupervised,klindt2026doeslejepalearnworld}. Nevertheless, the residual linear mixing can still entangle invariant and variant directions and therefore does not identify their decomposition. Consequently, whether the invariant-variant structure can be identified without reconstruction still remains an open question.

We introduce SplitJEPA, a JEPA that provably learns separate invariant and variant representations of a latent world without reconstruction. Our key insight is that prediction in representation space and consistency across invariant pairs resolve distinct levels of latent ambiguity. We achieve prediction in latent space, while block consistency under sufficient variation resolves the mixing between invariant and variant factors. We establish conditions under which \modelname identifies both latent blocks up to within-block orthogonal transformations, using only predictive and invariant correspondences without reconstructing observations. We further show that downstream predictions based on the invariant block are insensitive to variant perturbations when other inputs remain fixed. These results guide a framework that combines latent prediction, invariant-block agreement, and representation normalization. Experiments on controlled nonlinear systems and robotic manipulation tasks validate the proposed theory and demonstrate principled block recovery and improved downstream control.

\begin{figure}[!t]
    \centering
    \includegraphics[width=\textwidth]{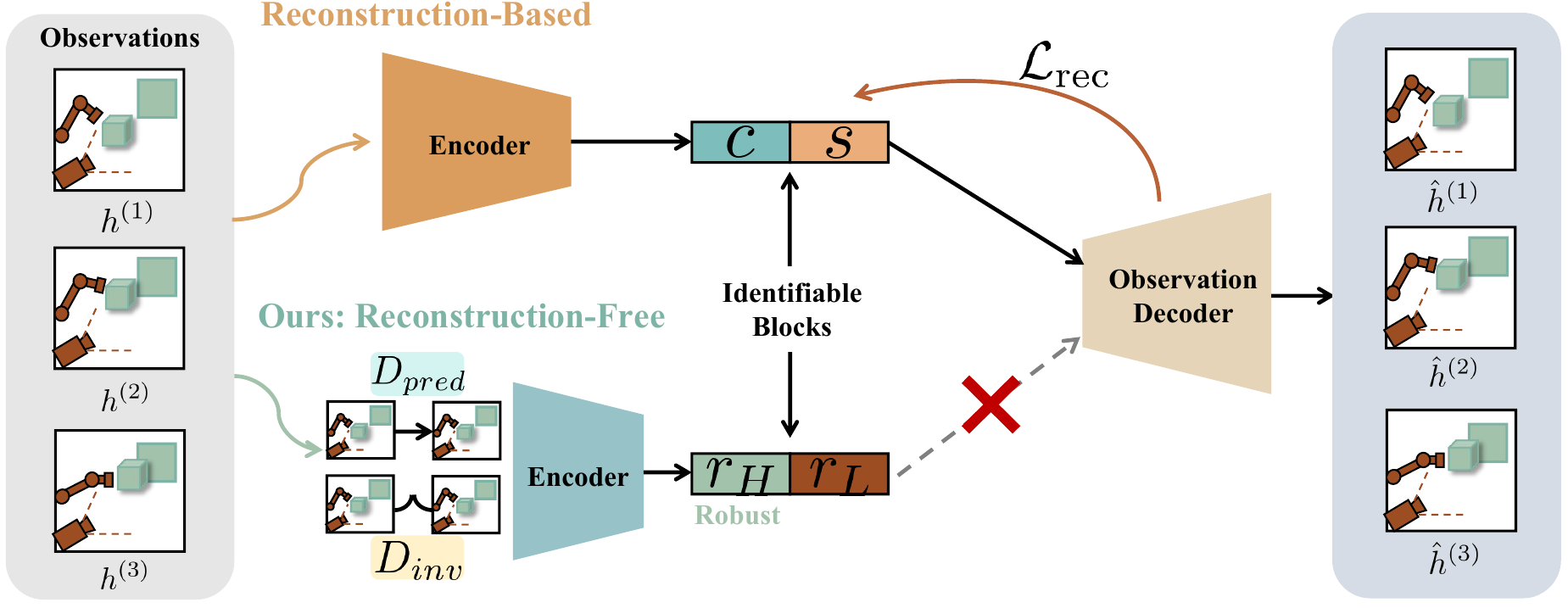}
    \setlength{\belowcaptionskip}{-15pt}
    \caption{\textbf{Block identifiability without observation reconstruction.}
Reconstruction-based identifiable models recover structured latent factors through an encoder-decoder pathway together with additional structural assumptions. \modelname combines prediction in latent space with block consistency to recover invariant and variant structure directly in representation space.}
    \label{fig:intro_comparison}
    \vspace*{-5pt}
\end{figure}
\vspace*{-10pt}

\section{Problem Background and Formulation}
\label{sec:background and formulation}
We consider a latent world observed through an unknown nonlinear process, 
\begin{equation}
h=f(z),\qquad z=\begin{bmatrix}z_H\\z_L\end{bmatrix}\in\mathbb R^d,\qquad d=d_H+d_L,
\end{equation} 
where $z_H\in\mathbb R^{d_H}$ and $z_L\in\mathbb R^{d_L}$ denote factors that remain shared and factors that may change across related observations. An encoder $\phi:\mathcal H\rightarrow\mathbb R^d$ induces
\begin{equation}
r=\phi(h)=g(z),\qquad g=\phi\circ f,\qquad
r=\begin{bmatrix} r_H\\r_L \end{bmatrix}.
\end{equation}
We study whether $r_H$ and $r_L$ recover the corresponding latent blocks without reconstructing $h$.

Following nonlinear identifiability settings, we take $f$ to be a smooth embedding on the connected latent support and $g$ to be continuously differentiable~\citep{kugelgen2021selfsupervised,daunhawer2023identifiabilityresultsmultimodalcontrastive}. This preserves latent-state information while leaving its parameterization undetermined, the central ambiguity of nonlinear
latent-variable models~\citep{hyvarinen1999nonlinear}.

\paragraph{Predictive structure.}
Temporal and contrastive identifiability results constrain latent parameterizations through relations between
observations~\citep{hyvarinen2016unsupervised,hyvarinen2017nonlinear,zimmermann2022contrastivelearninginvertsdata}. For LeJEPA, this connection yields linear recovery under stationary Gaussian latent dynamics~\citep{klindt2026doeslejepalearnworld,balestriero2025lejepaprovablescalableselfsupervised}. We build on this predictive setting:
\begin{equation}
\label{eq:predictive_dynamics}
z^{(1)}\mid X\sim\mathcal N(0,I_d),\qquad z^{(2)}=\rho_X z^{(1)}+\sqrt{1-\rho_X^2}\,\epsilon,
\end{equation}
where $0<\rho_X<1$, $\epsilon\mid X\sim\mathcal N(0,I_d)$,
and $\epsilon\perp z^{(1)}\mid X$. We denote this distribution by $\mathcal D_{\mathrm{pred}}$. The context $X$ controls the predictive correlation, while both states have the same marginal distribution.

\paragraph{Invariant correspondences.} Content-style learning supplies a complementary distinction: paired views preserve content while style may change~\citep{kugelgen2021selfsupervised}.
Multiview and multimodal identifiability results similarly use shared factors to link observations with view-specific
variation~\citep{lyu2022understanding,daunhawer2023identifiabilityresultsmultimodalcontrastive}.  Following this principle, let $\mathcal D_{\mathrm{inv}}$ denote a distribution over pairs
\begin{equation}
    z^{(1)}=(z_H,z_L^{(1)}),\qquad
z^{(2)}=(z_H,z_L^{(2)}).
\end{equation}
The learner observes only $\bigl(f(z^{(1)}),f(z^{(2)})\bigr)$. The invariance is specific to this correspondence: $z_H$ is shared within an invariant pair but may change across predictive pairs. Whereas $\mathcal D_{\mathrm{pred}}$
constrains predictive dependence, $\mathcal D_{\mathrm{inv}}$ gives the partition its meaning; the extent of variation within these correspondences determines whether they suffice to identify it.

\paragraph{Identification target.} 
Prior content-style and multimodal results identify shared factors up to a within-block invertible transformation~\citep{kugelgen2021selfsupervised,daunhawer2023identifiabilityresultsmultimodalcontrastive}. We seek the block-orthogonal form
\begin{equation}
\label{eq:block_target}
    g_H(z)=A_{11}z_H,\qquad
g_L(z)=A_{22}z_L,
\end{equation}
where $A_{11}^\top A_{11}=I_{d_H}$ and
$A_{22}^\top A_{22}=I_{d_L}$. This identifies the two subspaces while allowing coordinates to mix within each block. The total latent dimension and both block dimensions are taken as given.

Recovering this organization is stronger than retaining the complete latent state.  The following construction illustrates the distinction at the level of observations alone.

\begin{theorem}[Nonlinear Latent Ambiguity]
\label{th1}
Let $h=f(z_H,z_L)$. For any smooth nonlinear function $\psi:\mathbb{R}^{d_L}\rightarrow\mathbb{R}^{d_H}$, define
$\tilde z_H=z_H+\psi(z_L),\tilde z_L=z_L$,
and $\tilde f(\tilde z_H,\tilde z_L) = f\bigl(\tilde z_H-\psi(\tilde z_L),\tilde z_L\bigr)$.
Then $\widetilde f(\widetilde z)=f(z)$, while $\frac{\partial\widetilde z_H}{\partial z_L}=J_\psi(z_L)$.
Whenever this derivative is nonzero, the alternative first block depends on the variant factor. 
\end{theorem}

\begin{wrapfigure}{r}{0.48\textwidth}
    \centering
    \includegraphics[width=\linewidth]{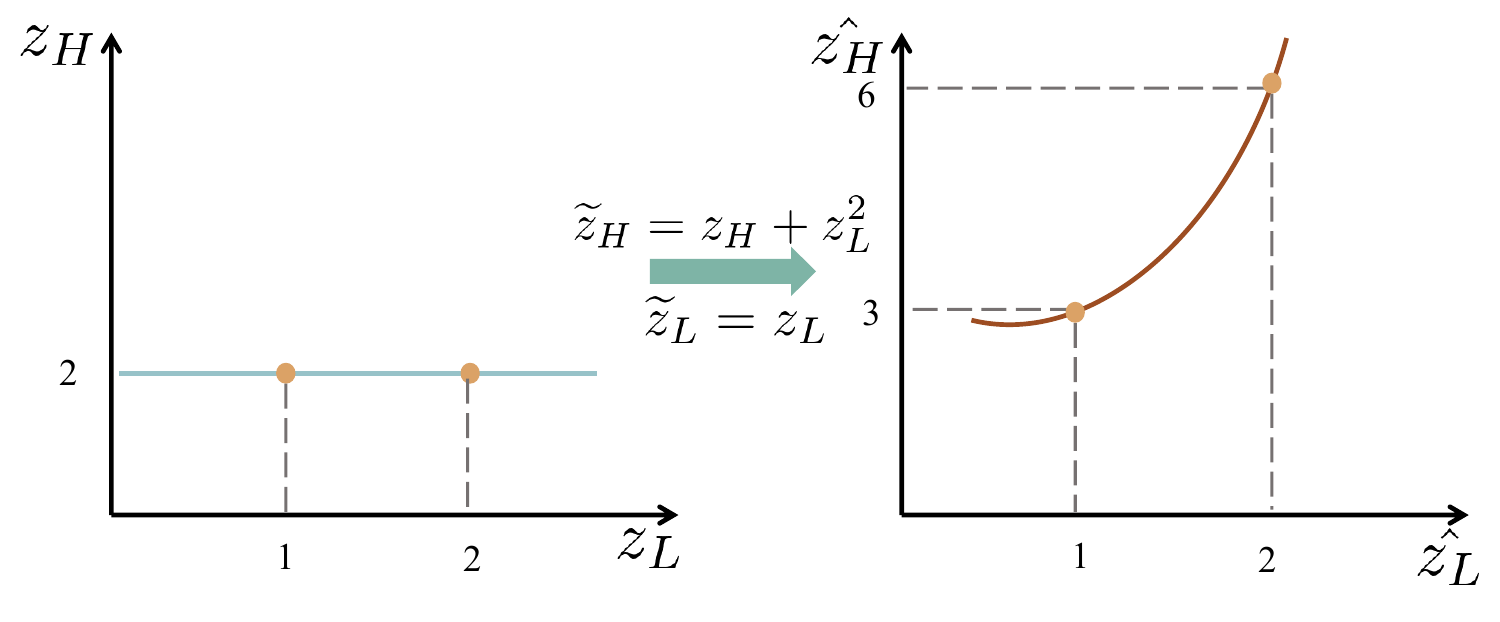}
    \caption{Nonlinear mixing of the latent blocks.}
    \label{fig:example1}
    \vspace{-30pt}
\end{wrapfigure}

\paragraph{Structure.} Letting $\tilde z=V_{\mathrm{leak}}(z)$, the transformation has Jacobian $J_{V_{\mathrm{leak}}}(z)
=
\begin{bmatrix}
I_{d_H} & J_\psi(z_L)\\
0 & I_{d_L}
\end{bmatrix}$. Its off-diagonal block introduces variant information into the first block, while the explicit inverse preserves the complete latent state.

\begin{example} Let $d_H=d_L=1$ and $\psi(z_L)=z_L^2$. The states $(z_H,z_L)=(2,1)$ and $(2,2)$ share the same invariant $z_H=2$. After the transformation, they become $(\tilde z_H,\tilde z_L)=(3,1)$ and $(6,2)$. Thus, the transformed first coordinate varies with $z_L$, though the original invariant factor is unchanged.
\end{example}

\begin{proposition}[Ambiguity under the Basic Setup]
\label{prop1}
Under the smooth observation model above, the invariant-variant decomposition is not identifiable from observations alone. The parameterizations $(f,z)$ and $(\tilde f,\tilde z)$ in Theorem~\ref{th1} are observationally equivalent but are not related by independent within-block transformations when $J_\psi$ is nonzero somewhere on the latent support. Consequently, information preservation does not imply recovery of the block-structured target.
\end{proposition}

This ambiguity concerns observation equivalence, where the transformed parameterization need not preserve the relational structure specified above for all such admissible transformations.
Recent work similarly studies which latent structures remain
recoverable when full nonlinear identification is unavailable, using diversity in latent-observation dependencies as an identifying signal~\citep{zheng2026diverse}. Our setting uses a
different source of structure: prediction first reduces the ambiguity to a global orthogonal transformation, and variation across correspondences identifies its block partition.

\section{SplitJEPA: Learning Invariant and Variant Latent Worlds without Reconstruction}
\modelname resolves two levels of ambiguity: predictive identification recovers the complete latent state, and sufficiently diverse correspondences determine its invariant-variant partition. We first characterize the ambiguity remaining after predictive identification (Section~\ref{sec:predictive identification}), then show how sufficient variation resolves the latent partition (Section~\ref{sec:block identification}). We subsequently discuss the resulting robustness (Section~\ref{sec:perturbation robustness}) and formulate the estimation conditions (Section~\ref{sec:estimation}).

\subsection{Predictive Identification and the Remaining Ambiguity}
\label{sec:predictive identification}

JEPA learns from statistical relations between observations without reconstructing
them~\citep{assran2023selfsupervisedlearningimagesjointembedding,bardes2024revisiting}. LeJEPA provides a reconstruction-free instance of this principle~\citep{balestriero2025lejepaprovablescalableselfsupervised}. Under the predictive setting in Section~\ref{sec:background and formulation}, its identifiability analysis characterizes global latent recovery through a linear representation~\citep{klindt2026doeslejepalearnworld}. For centered and whitened population solutions, this takes the
form $g(z)=Qz$, where $Q^\top Q=I_d$. This result excludes the nonlinear mixing in Theorem~\ref{th1}, but a global rotation may still mix the two latent blocks.

\begin{proposition}[Block Ambiguity of LeJEPA]
\label{prop2}
Suppose a representation belongs to the globally identified equivalence class, so that $g(z)=Qz, Q^\top Q=I_d$. For any orthogonal matrix $R\in\mathbb R^{d\times d}$, the representation $\tilde g(z)=Rg(z)$ belongs to the same equivalence class. Since this class contains transformations with nonzero cross-block components, predictive identification alone does not determine the
invariant-variant partition.
\end{proposition}

Orthogonal transformations preserve distances, covariance, and the isotropic predictive structure. As a result, global recovery does not select the orientation associated with the correspondence defining $z_H$. Nevertheless, reducing the original nonlinear ambiguity to an unknown linear transformation substantially simplifies the remaining problem.

\begin{proposition}[Recovery up to Orthogonal Transformation]
Under global predictive identification in the centered and whitened representation class, the learned partition admits
\label{prop3}
\begin{equation}
\label{eq:linear_partition}
r=Qz
=
\begin{bmatrix}
A_{11} & A_{12}\\
A_{21} & A_{22}
\end{bmatrix}
\begin{bmatrix}
z_H\\
z_L
\end{bmatrix},
\qquad Q^\top Q=I_d.
\end{equation}
The complete latent state is linearly recoverable as $z=Q^\top r$. The block-orthogonal target in Equation~\ref{eq:block_target}
holds if and only if $A_{12}=0$ and $A_{21}=0$.
\end{proposition}

It locates the remaining ambiguity in
$A_{12}$ and $A_{21}$. Predictive recovery preserves the complete state, but it does not determine which directions should remain unchanged under the correspondence defining $z_H$. The question is whether observed changes suffice to determine the desired partition.

\subsection{Identifying the Partition through Variation}
\label{sec:block identification}
\begin{figure}[t]
    \centering
    \includegraphics[width=\textwidth]{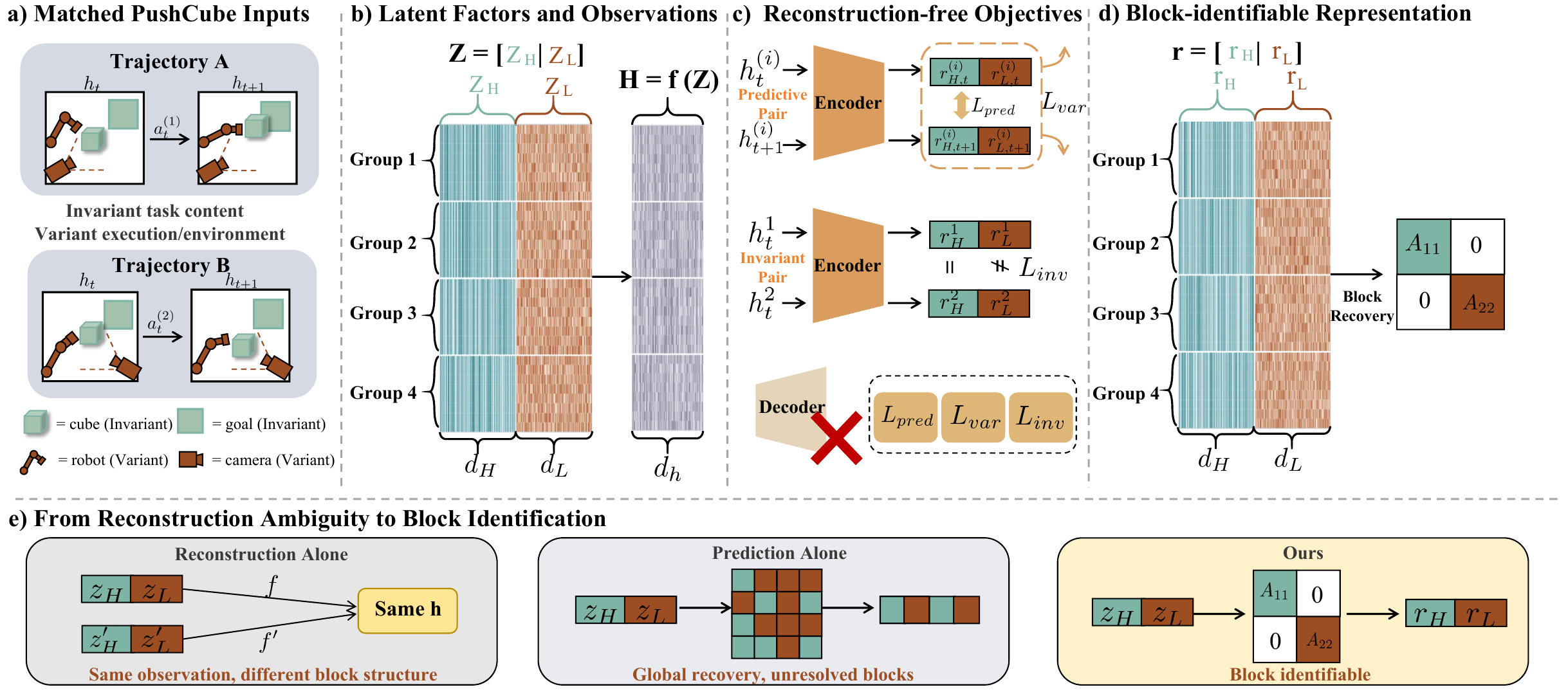}
    \caption{ \textbf{Overview of reconstruction-free block identifiability.}
    (a) Matched PushCube trajectories share invariant task content while varying in execution and environmental factors.
    (b) The ground-truth latent state consists of invariant and variant blocks that jointly generate the observations.
    (c) Predictive alignment and variance preservation identify the complete latent state, while block consistency determines its partition without an observation decoder.
    (d) The resulting representation recovers the invariant and variant subspaces.
    (e) Reconstruction alone permits observationally equivalent latent parameterizations, and prediction alone leaves a global orthogonal ambiguity; additional consistency structure resolves this ambiguity and yields reconstruction-free block identification.}
    \label{fig:theory_overview}
    \vspace{-17pt}
\end{figure}

Once the representation is globally linear, agreement across correspondences constrains its dependence on the variant state. The strength of this constraint depends on which directions actually change. We formalize the required coverage below.
\begin{assumption}[Sufficient Variation]
\label{assump:sufficient_variation}
For $(z^{(1)},z^{(2)})\sim\mathcal D_{\mathrm{inv}}$, let $\Delta z_L=z_L^{(1)}-z_L^{(2)}$. The second-moment matrix of the variant changes satisfies
\begin{equation}
\Sigma_{\Delta L}
=
\mathbb E_{\mathcal D_{\mathrm{inv}}}
\left[ \Delta z_L\Delta z_L^\top \right] \succ0.
\end{equation}
\end{assumption}
Positive definiteness means that $\mathbb E[(v^\top\Delta z_L)^2]>0$ for every nonzero
$v\in\mathbb R^{d_L}$: no linear direction in the variant block remains unchanged across all available pairs. A single pair need not vary every coordinate, and the coordinates of $\Delta z_L$ need not be independent.

\paragraph{Intuition.}
Predictive learning first recovers the complete latent state up to a global orthogonal transformation $g^\star(z)=Qz$. After partitioning $Q$ according to the invariant and variant latent blocks, the off-diagonal matrices $A_{12}$ and $A_{21}$ describe the remaining cross-block mixing. For matched pairs sharing $z_H$, the invariant-representation difference is $g_H^\star(z^{(1)})-g_H^\star(z^{(2)})=A_{12}\Delta z_L$. If the variant differences have a full-rank second moment, agreement of $g_H^\star$ across the matched pairs forces $A_{12}=0$. The orthogonality of $Q$ then forces $A_{21}=0$, leaving a block-diagonal transformation.

\begin{theorem}[Reconstruction-free Block Identifiability]
\label{th2}
Under the predictive model in Equation~\ref{eq:predictive_dynamics} and
Assumption~\ref{assump:sufficient_variation}, let $g^\star$ be a global population minimizer of
$\mathcal L_{\mathrm{pred}}$ over the centered and whitened representation class specified in  Section~\ref{sec:estimation}.
If $g_H^\star(z^{(1)})=g_H^\star(z^{(2)})$ almost surely under $\mathcal D_{\mathrm{inv}}$, then
\begin{equation}
g^\star(z)=Qz
=
\begin{bmatrix}
A_{11} & 0\\
0 & A_{22}
\end{bmatrix}
\begin{bmatrix}
z_H\\
z_L
\end{bmatrix},
\end{equation}
where $A_{11}^\top A_{11}=I_{d_H}$ and
$A_{22}^\top A_{22}=I_{d_L}$. Hence, the invariant and variant latent subspaces are
identified up to independent within-block orthogonal transformations without reconstruction.
\end{theorem}

\paragraph{Identifying what can change.}
An invariant representation ignores every change exposed by the correspondence. However, agreement on a limited set of changes may leave other variant directions undetected. A correspondence is informative not merely because it preserves $z_H$, but because its variation covers the directions in $z_L$. Sufficient variation ensures that every nonzero linear dependence on $z_L$ is exposed somewhere in the pair distribution. This condition becomes significant after global predictive recovery: the remaining dependence on $z_L$ is described by a matrix.

\paragraph{Interpretation.}
Theorem~\ref{th2} separates recovery of the complete latent state from identification of its organization. Predictive information determines which latent information is retained, while agreement across matched correspondences determines how that information is partitioned into invariant and variant subspaces. The resulting identification is block-wise: invariant factors may remain mixed within $r_H$, and variant factors may remain mixed within $r_L$. What is excluded is mixing across the two blocks. Thus, the theorem identifies the invariant-variant decomposition up to independent within-block isometries, without requiring an observation decoder or access to the ground-truth latent variables. The proof is provided in Appendix~\ref{app:proof_th2}.

\begin{remark}[Stability under approximate invariant consistency]
The sufficient-variation condition also provides a quantitative stability guarantee when invariant consistency is approximate. Suppose that global recovery gives $g(z)=Qz$ with $Q$ orthogonal, and that the matched-pair consistency error satisfies $\mathbb{E}\!\left[\left\|g_H(z^{(1)})-g_H(z^{(2)})\right\|_2^2\right]\leq \varepsilon$. If $\Sigma_{\Delta L}=\mathbb{E}\!\left[\Delta z_L\Delta z_L^\top\right]\succeq \mu I$ for some \(\mu>0\), then $\varepsilon\geq\operatorname{tr}\!\left(A_{12}\Sigma_{\Delta L}A_{12}^\top\right)\geq\mu\|A_{12}\|_F^2$. Consequently, $\|A_{12}\|_F\leq\sqrt{\frac{\varepsilon}{\mu}}$.    
\end{remark}

\begin{example}
\label{ex:complementary_changes}
Let $z_H=c$ and $z_L=(u,v)^\top$. For fixed $c$, varying $(u,v)$ traces a two-dimensional latent sheet. One family of correspondences changes $u$, while another changes $v$, possibly at different values of $c$. Although neither spans the variant plane alone, their combined changes do. Since prediction in latent space constrains $r$ to be an orthogonal transform of $z$, agreement of $r_H$ along both directions removes its dependence on $z_L$; orthogonality assigns the complementary variant subspace to $r_L$.
\end{example}
\begin{wrapfigure}{r}{0.4\textwidth}
    \centering
    \includegraphics[width=\linewidth]{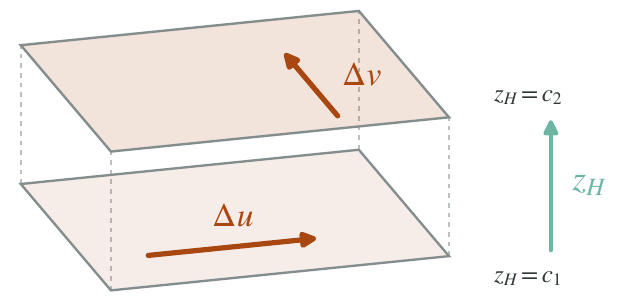}
    \caption{Complementary changes span the variant subspace.}
    \label{fig:example2}
    \vspace{-10pt}
\end{wrapfigure}

\paragraph{Learning from separate changes.}
Example~\ref{ex:complementary_changes} illustrates the role of complementary evidence. A direction that remains unchanged in one collection may change in another. The differences determine which directions can belong to the invariant representation. A single difference generally rules out only a subset of possible cross-block transformations. Different pairs contribute different constraints. What matters is whether the collection of differences spans the variant subspace. This condition is relevant in practical settings. For example, in robot manipulation, if viewpoint and lighting always change in a fixed linear combination, another combination may remain
unchanged and therefore indistinguishable from invariant content. Thus, the condition concerns the directions revealed by the changes and their aggregate coverage.

\paragraph{Why sufficient variation is plausible.}
Complementary changes can arise from combining repeated measurements of invariant content. In a robotic environment, observations from different cameras and lighting conditions expose different visual directions. Matched robotic states may additionally differ in arm configuration or approach, provided that the factors assigned to $z_H$ remain fixed.  Related to this, different policies have been used to recover environmental features in identifiable skill learning~\citep{reizinger2026skill}. Here, diversity plays a more specific role: individual sources need not be sufficient, but their combined changes within correspondences that preserve $z_H$ must span the variant subspace.
The relevant diversity is therefore measured within the specified correspondence. Observations must preserve the shared state while revealing different directions of variation. The learner need not observe these latent changes directly; they only need to be present in the paired data. This also explains the role of Assumption~\ref{assump:sufficient_variation}. A variant direction that never changes across the available pairs cannot be distinguished from invariant content. 
A direction that changes only weakly provides less reliable evidence in finite samples. Sufficient variation ensures that every variant direction contributes information for determining the partition.

\subsection{Robustness under Variant Perturbations}
\label{sec:perturbation robustness}
Block identifiability is not only a representation-level guarantee; it also has an operational consequence for downstream prediction. When a downstream computation uses only the invariant representation, changes confined to the variant latent factor have no path through its representation.
\begin{theorem}[Gradient Isolation]
\label{th3}
Under the block identifiability established in Theorem~\ref{th2}, suppose a differentiable downstream predictor depends solely on the invariant representation $r_H$.
Then the prediction is locally insensitive to perturbations in $z_L$: $\frac{\partial Y}{\partial z_L}=0$. Consequently, local variant perturbations cannot propagate
through the invariant representation.
\end{theorem}

\paragraph{Interpretation.}
Theorem~\ref{th3} gives block identifiability an operational meaning:
a predictor using $r_H$ can respond to changes in the shared state
without being affected by how the variant factors realize that state.
Indeed, $g_H(z)=A_{11}z_H$ implies that all admissible states with
the same $z_H$ produce the same invariant representation, even when their observations differ substantially. The zero gradient is the local expression of this invariance. The invariance guarantee is specific to variant changes, which do not require insensitivity to changes in $z_H$. Once $g_H$ depends only on $z_H$, a change in $z_L$ has no path into the invariant representation. The result applies to finite variant perturbations, as long as the perturbed state remains within the modeled support. 
The proof is provided in
Appendix~\ref{app:proof_th3}.

\begin{example}
\label{ex:isolated_channel}
Continuing Example~\ref{ex:complementary_changes}, fix a sheet $z_H=c$ and the downstream context, and move locally along
$z(t)=(c,u+t\delta_u,v+t\delta_v)$. The observation and $r_L$ may change, but Theorem~\ref{th2} gives $r_H(t)=A_{11}c$ along this path. Hence any predictor $Y(t)=\Psi(\mathrm{Context},r_H(t))$ has zero directional derivative at $t=0$ for every variant direction $(\delta_u,\delta_v)$.
\end{example}

\paragraph{Task-relative robustness.}
In robotic manipulation, changes in camera pose or illumination leave predictions based on $r_H$ unchanged when they affect only $z_L$. If a perturbation also modifies the shared task content, the invariance guarantee no longer applies, and the resulting prediction change depends on the magnitude of the change in the invariant factor and the sensitivity of the downstream map. This guarantee is relative to the specified correspondence: a camera change may be variant for object recognition but relevant for viewpoint estimation. Perturbations that also change $z_H$ fall outside this invariance guarantee. With imperfect block recovery, downstream sensitivity additionally depends on residual cross-block dependence and the predictor's sensitivity.

\subsection{Estimation and Normalization}
\label{sec:estimation}
We now translate these theoretical conditions into representation-space objectives. For pairs drawn from the corresponding relational distribution, write $g=\phi\circ f$, with $g(z^{(i)})=\phi(h^{(i)})$ for each observed pair. The losses are evaluated using the observed pairs alone.

\paragraph{Representation and consistency losses.}
We use
\begin{equation}
\mathcal L_{\mathrm{pred}}(g)
=
\frac12\,
\mathbb E_{\mathcal D_{\mathrm{pred}}}
\left[
\lVert g(z^{(1)})-g(z^{(2)})\rVert_2^2
\right], \quad
\mathcal L_{\mathrm{inv}}(g)
=
\mathbb E_{\mathcal D_{\mathrm{inv}}}
\left[
\lVert g_H(z^{(1)})-g_H(z^{(2)})\rVert_2^2
\right].
\label{eq:representation_losses}
\end{equation}
For estimation, predictive alignment and invariant consistency are combined with a variance-preserving regularizer:
\begin{equation}
\mathcal L_{\mathrm{repr}}(g)
=
\mathcal L_{\mathrm{pred}}(g)
+
\lambda_{\mathrm{inv}}\mathcal L_{\mathrm{inv}}(g)
+
\lambda_{\mathrm{var}}\mathcal L_{\mathrm{var}}(g),
\qquad
\lambda_{\mathrm{inv}},\lambda_{\mathrm{var}}>0.
\end{equation}
The variance regularizer discourages collapse, while $\mathcal L_{\mathrm{inv}}$ encourages agreement of the designated invariant block across matched pairs. Empirical versions replace expectations with sample averages.

\paragraph{Normalization constraint.}
The exact analysis additionally restricts the full representation under the predictive marginal to satisfy
\begin{equation}
\mathbb E[g(z)]=0,\qquad
\operatorname{Cov}(g(z))=I_d,
\qquad z\sim\mathcal N(0,I_d).
\label{eq:normalization_constraint}
\end{equation}
Finite-weight variance regularization encourages this normalized geometry but does not by itself impose exact whitening. Full whitening includes zero cross-covariance between the designated blocks. Once predictive optimality yields $g^\star(z)=Qz$, normalization gives $QQ^\top=I_d$. Since $Q$ is square, it is orthogonal, providing the geometry used to eliminate both cross-block components.

\section{Experiments}
We test two implications of the theory: whether correspondences resolve the ambiguity left by predictive recovery, and whether the resulting blocks support control in robotic manipulation under varying conditions. Synthetic experiments isolate the identification mechanism; robotic experiments evaluate a practical implementation across manipulation tasks and evaluation settings.

\subsection{Synthetic Evaluation}
\paragraph{Simulation Setup.}
We train a reconstruction-free predictive encoder on a nonlinear Gaussian OU system, and all methods share the encoder. We report Global $R^2$, Invariant and Variant $R^2$, cross-block leakage $\lVert A_{12}\rVert_F$ and $\lVert A_{21}\rVert_F$, and block purity over five seeds. Ground-truth latents are used only for evaluation. This shared-encoder comparison isolates organization of the recovered information from encoder quality. Data generation and estimation details appear in Appendix~\ref{app:simulation setting}.

\paragraph{Main Results.}
\begin{table}[t]
  \centering
  \small
  \setlength{\tabcolsep}{2.8pt}
  \caption{
    \textbf{Block recovery on the synthetic OU system.}
    All methods reuse the same frozen reconstruction-free encoder and differ only in the information available to the spectral recovery stage. Results are the mean $\pm$ standard deviation over five seeds. Global $R^2$ is approximately $0.9977\pm0.0016$ for all methods.
  }
  \label{tab:simulation-main}
    \resizebox{\textwidth}{!}{%
  \begin{tabular}{lccccc}
    \toprule
    Method
      & Invariant $R^2$
      & Variant $R^2$
      & $\lVert A_{12}\rVert_F$
      & $\lVert A_{21}\rVert_F$
      & Block Purity\\
    \midrule
    LeJEPA
      & $0.9613\pm0.0334$
      & $0.9612\pm0.0328$
      & $0.3292\pm0.2145$
      & $0.3299\pm0.2137$
      & $0.9635\pm0.0318$\\
    Shuffled Pairs
      & $0.7134\pm0.0714$
      & $0.7102\pm0.0736$
      & $1.0600\pm0.1415$
      & $1.0606\pm0.1394$
      & $0.7142\pm0.0734$\\
    \textbf{\modelname (Ours)}
      & $\mathbf{0.9979\pm0.0014}$
      & $\mathbf{0.9973\pm0.0019}$
      & $\mathbf{0.0031\pm0.0018}$
      & $\mathbf{0.0236\pm0.0021}$
      & $\mathbf{0.99993\pm0.00001}$\\
    \bottomrule
  \end{tabular}
  }
\end{table}
Table~\ref{tab:simulation-main} illustrates the distinction in
Proposition~\ref{prop2}: equally informative representations can
differ substantially in how they organize latent factors.
With the predictive encoder fixed, \modelname achieves near-perfect block recovery with little cross-block leakage, whereas LeJEPA retains mixing between the designated blocks. Shuffling the pairs disrupts this separation without changing the globally available
information. These results support the \modelname mechanism: meaningful correspondences resolve an ambiguity that global latent recovery alone leaves open.

\paragraph{Ablations and Scaling.}
\begin{figure}[t]
  \centering
  \includegraphics[width=\linewidth]{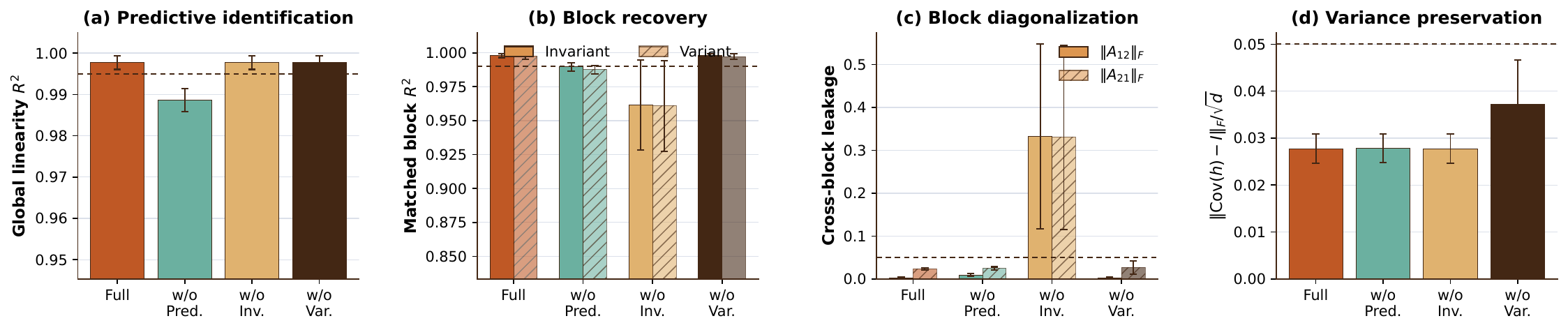}
  \caption{
    \textbf{Ablation of the predictive, invariance, and variance-preserving
    components.} Bars show mean $\pm$ standard deviation over five seeds. (a) Removing predictive alignment reduces global linear recovery. (b)--(c) Removing invariant-pair recovery leaves global identification unchanged but degrades block recovery and increases cross-block leakage. (d) Removing variance regularization and ZCA normalization increases the representation covariance error.
  }
  \vspace{-15pt}
  \label{fig:simulation-loss-ablation}
\end{figure}

Figure~\ref{fig:simulation-loss-ablation} distinguishes the roles
of prediction, correspondence and normalization. Removing prediction impairs global recovery, while removing invariant-pair recovery primarily disrupts the latent partition. Removing variance
regularization and ZCA normalization increases covariance error
and makes complementary-block recovery less stable.
Scaling illustrated in Figures~\ref{fig:scaling_global_r2}--\ref{fig:scaling_normalized_a21} reveals a corresponding distinction: more predictive training improves global recovery, whereas more recovery samples improve block separation even with the encoder fixed. Larger latent spaces make complementary-block estimation more sensitive, highlighting the finite-sample demands of recovering the full partition. Full results appear in Appendix~\ref{app:simulation}.

\subsection{Robotic Manipulation Evaluation}

\paragraph{Robotic manipulation Setup.}
We evaluate PickCube and PushCube in ManiSkill~\citep{tao2025maniskill}\footnote{\url{https://maniskill.readthedocs.io/en/latest/}} using 800 successful trajectories per task. Matched trajectories provide approximate cross-execution correspondences, while same-state base/side-camera views provide viewpoint correspondences. Behavior-cloning policies use either $r_H$ or $r_H+r_L$; we compare them with Raw, frozen R3M~\citep{nair2022rm}, and LeJEPA~\citep{balestriero2025lejepaprovablescalableselfsupervised}. We report validation action MSE and closed-loop SR@10cm under ID, camera, lighting, and combined visual shifts. Full data construction, training and evaluation details appear in Appendix~\ref{app:realworld setting}.
\vspace{-10pt}

\paragraph{Action prediction and control.}
\begin{figure}[t]
    \centering
    \begin{subfigure}[t]{0.49\textwidth}
        \centering
        \includegraphics[width=\linewidth]{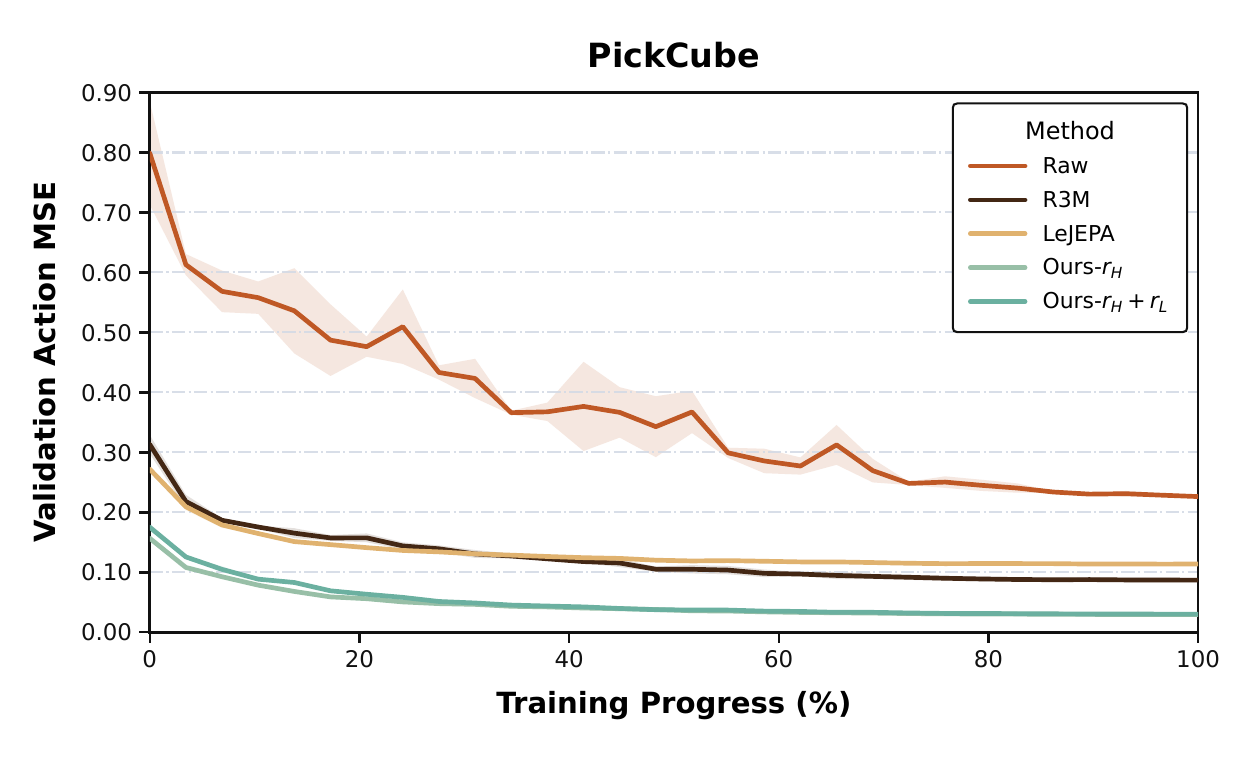}
        \label{fig:offline_mse_pickcube}
    \end{subfigure}
    \hfill
    \begin{subfigure}[t]{0.49\textwidth}
        \centering
        \includegraphics[width=\linewidth]{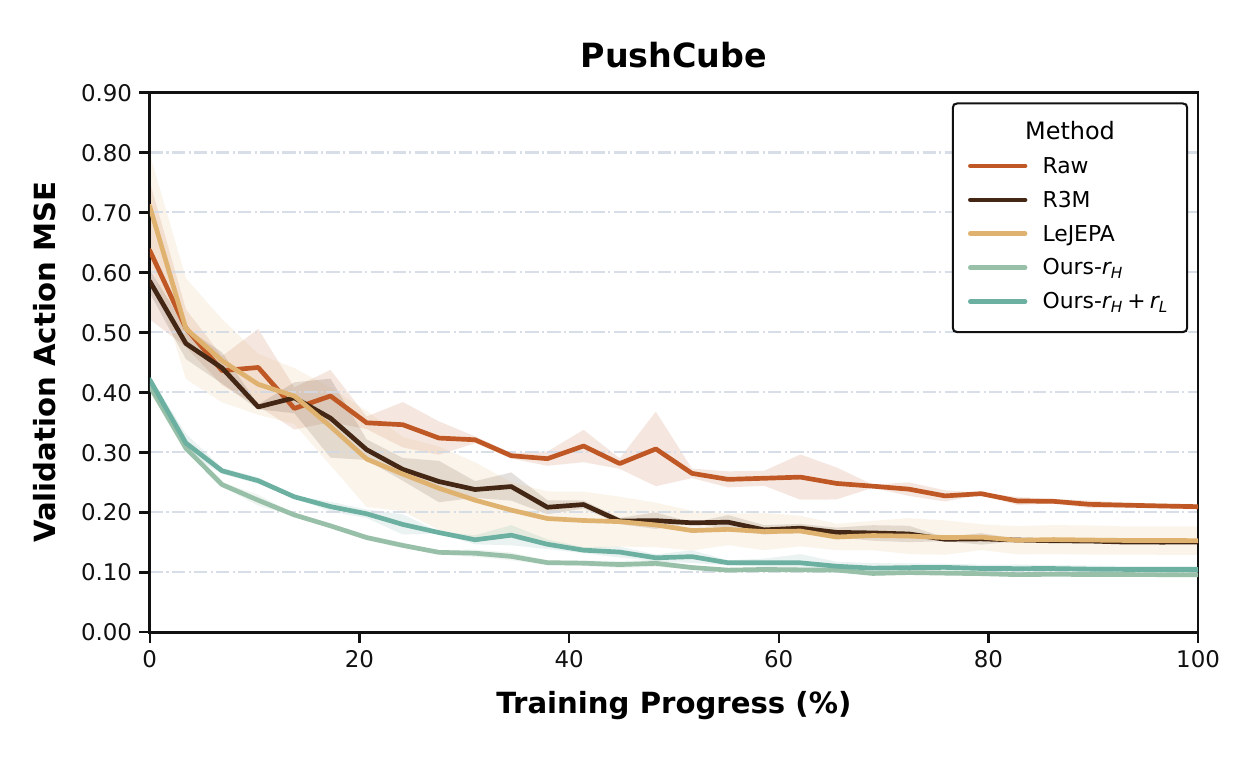}
        \label{fig:offline_mse_pushcube}
    \end{subfigure}
    \vspace{-20pt}

    \caption{\textbf{Offline action prediction performance.}
    Validation action MSE during offline behavior cloning training on PickCube and PushCube.}
    \label{fig:offline_mse}
\end{figure}

Both proposed representations improve validation action prediction over the baselines (Figure~\ref{fig:offline_mse}). Their similar errors suggest that $r_H$ retains much of the information needed
for this task, despite excluding the designated variant block.

The clearest closed-loop benefits appear under visual shifts
(Table~\ref{tab:realworld main}), where camera and lighting changes
alter observations without changing the physical control objective.
The gains are especially pronounced on PushCube; PickCube remains more challenging, indicating that separating visual variation
does not remove the demands of grasping and transport. The relative performance of $r_H$ and $r_H+r_L$ also depends on the
shift: retaining both blocks benefits some conditions.

\begin{table*}[t]
  \centering
  \caption{\textbf{Robotic manipulation  performance.}
  Results on PickCube and PushCube under ID and OOD conditions. CONTACT and GRASP are reported for ID; all OOD columns report SR@10cm.}
  \label{tab:realworld main}
  \scriptsize
  \setlength{\tabcolsep}{3.2pt}
  \renewcommand{\arraystretch}{1.05}
  \resizebox{\textwidth}{!}{%
  \begin{tabular}{lccccc|ccccc}
    \toprule
    & \multicolumn{5}{c|}{\textbf{PushCube Manipulation}}
    & \multicolumn{5}{c}{\textbf{PickCube Manipulation}} \\
    
    \cmidrule(lr){2-6} \cmidrule(lr){7-11}
    
    Method
      & \multicolumn{2}{c}{ID}
      & \multicolumn{3}{c|}{OOD}
      & \multicolumn{2}{c}{ID}
      & \multicolumn{3}{c}{OOD} \\
      
    \cmidrule(lr){2-3} \cmidrule(lr){4-6}
    \cmidrule(lr){7-8} \cmidrule(lr){9-11}
    
      & CONTACT
      & SR@10cm
      & Camera
      & Lighting
      & Visual
      & GRASP
      & SR@10cm
      & Camera
      & Lighting
      & Visual \\
    \midrule

    Raw
      & $100.0\!\pm\!0.0$
      & $98.0\!\pm\!1.0$
      & $12.8\!\pm\!6.8$
      & $98.3\!\pm\!1.5$
      & $12.7\!\pm\!7.0$
      & $23.0\!\pm\!1.6$
      & $14.0\!\pm\!1.4$
      & $10.3\!\pm\!0.5$
      & $11.7\!\pm\!0.9$
      & $10.7\!\pm\!0.5$ \\

    R3M
      & $100.0\!\pm\!0.0$
      & $97.7\!\pm\!1.2$
      & $48.0\!\pm\!7.5$
      & $90.7\!\pm\!2.1$
      & $25.8\!\pm\!1.8$
      & $73.0\!\pm\!1.6$
      & $32.0\!\pm\!1.4$
      & $14.3\!\pm\!0.5$
      & $10.0\!\pm\!0.0$
      & $10.0\!\pm\!0.0$ \\

    LeJEPA
      & $90.3\!\pm\!14.2$
      & $67.3\!\pm\!20.1$
      & $25.0\!\pm\!2.6$
      & $31.0\!\pm\!31.2$
      & $28.7\!\pm\!17.5$
      & $33.3\!\pm\!3.3$
      & $16.3\!\pm\!0.9$
      & $10.0\!\pm\!0.0$
      & $10.0\!\pm\!0.0$
      & $10.0\!\pm\!0.0$ \\

    \textbf{Ours-$r_H$}
      & $100.0\!\pm\!0.0$
      & $99.3\!\pm\!0.6$
      & $71.5\!\pm\!10.4$
      & $\mathbf{99.7\!\pm\!0.6}$
      & $64.5\!\pm\!12.7$
      & $\mathbf{97.0\!\pm\!1.4}$
      & $45.0\!\pm\!2.4$
      & $\mathbf{30.3\!\pm\!6.1}$
      & $42.7\!\pm\!4.5$
      & $28.3\!\pm\!3.9$ \\

    \textbf{Ours-$r_H+r_L$}
      & $100.0\!\pm\!0.0$
      & $99.3\!\pm\!0.6$
      & $\mathbf{75.3\!\pm\!8.1}$
      & $98.3\!\pm\!1.5$
      & $\mathbf{70.2\!\pm\!12.9}$
      & $96.7\!\pm\!0.5$
      & $\mathbf{46.3\!\pm\!4.9}$
      & $27.7\!\pm\!0.9$
      & $\mathbf{43.3\!\pm\!2.9}$
      & $\mathbf{30.3\!\pm\!1.7}$ \\

    \bottomrule
  \end{tabular}
  }
  \vspace{-10pt}
\end{table*}

\paragraph{Ablation and Robustness Experiments.}
Table~\ref{tab:representation_ablation} tests prediction in latent space, invariance and cross-trajectory diversity. Removing prediction reduces performance under camera and combined shifts, despite relatively small ID changes on PushCube. Removing invariance produces the largest PushCube degradation: Camera SR@10cm falls while ID success remains high. PickCube also loses combined-shift performance, although camera and lighting effects are less uniform. The diversity ablation reveals a task-dependent trade-off: removing the diversity mechanism is neutral or beneficial in several PushCube conditions, but reduces PickCube camera and combined-shift success for both representation inputs. In contrast, the population guarantee requires informative variation among exact invariant correspondences, and practical cross-trajectory pairing trades broader variation coverage against approximate alignment quality. Additional experiments appear in Appendix~\ref{app:realworld-robustness}.

\section{Conclusion}
We present SplitJEPA, a reconstruction-free framework for recovering invariant and variant latent structure. We show that prediction in representation space and invariant-block consistency jointly yield block identifiability under sufficient variation, and that the recovered structure isolates computations based on the invariant block from variant perturbations. Experiments on controlled nonlinear systems and robotic manipulation tasks demonstrate improved block recovery, prediction and control. Their scale is the main limitation of this work, and scaling SplitJEPA to large-scale complex real-world scenarios is an important next step. More broadly, our results suggest that representation learning can recover not only the information in a latent world but also how it is organized.

\bibliography{iclr2027_conference}

\begin{thebibliography}{33}
\providecommand{\natexlab}[1]{#1}
\providecommand{\url}[1]{\texttt{#1}}
\expandafter\ifx\csname urlstyle\endcsname\relax
  \providecommand{\doi}[1]{doi: #1}\else
  \providecommand{\doi}{doi: \begingroup \urlstyle{rm}\Url}\fi

\bibitem[Assran et~al.(2023)Assran, Duval, Misra, Bojanowski, Vincent, Rabbat, LeCun, and Ballas]{assran2023selfsupervisedlearningimagesjointembedding}
Mahmoud Assran, Quentin Duval, Ishan Misra, Piotr Bojanowski, Pascal Vincent, Michael Rabbat, Yann LeCun, and Nicolas Ballas.
\newblock Self-supervised learning from images with a joint-embedding predictive architecture.
\newblock \emph{arXiv preprint arXiv:2301.08243}, 2023.

\bibitem[Assran et~al.(2025)Assran, Bardes, Fan, Garrido, Howes, Mojtaba, Komeili, Muckley, Rizvi, Roberts, Sinha, Zholus, Arnaud, Gejji, Martin, Hogan, Dugas, Bojanowski, Khalidov, Labatut, Massa, Szafraniec, Krishnakumar, Li, Ma, Chandar, Meier, LeCun, Rabbat, and Ballas]{assran2025vjepa2selfsupervisedvideo}
Mido Assran, Adrien Bardes, David Fan, Quentin Garrido, Russell Howes, Mojtaba, Komeili, Matthew Muckley, Ammar Rizvi, Claire Roberts, Koustuv Sinha, Artem Zholus, Sergio Arnaud, Abha Gejji, Ada Martin, Francois~Robert Hogan, Daniel Dugas, Piotr Bojanowski, Vasil Khalidov, Patrick Labatut, Francisco Massa, Marc Szafraniec, Kapil Krishnakumar, Yong Li, Xiaodong Ma, Sarath Chandar, Franziska Meier, Yann LeCun, Michael Rabbat, and Nicolas Ballas.
\newblock V-jepa 2: Self-supervised video models enable understanding, prediction and planning.
\newblock \emph{arXiv preprint arXiv:2506.09985}, 2025.

\bibitem[Balestriero \& LeCun(2025)Balestriero and LeCun]{balestriero2025lejepaprovablescalableselfsupervised}
Randall Balestriero and Yann LeCun.
\newblock Lejepa: Provable and scalable self-supervised learning without the heuristics.
\newblock \emph{arXiv preprint arXiv:2511.08544}, 2025.

\bibitem[Bardes et~al.(2024)Bardes, Garrido, Ponce, Chen, Rabbat, LeCun, Assran, and Ballas]{bardes2024revisiting}
Adrien Bardes, Quentin Garrido, Jean Ponce, Xinlei Chen, Michael Rabbat, Yann LeCun, Mido Assran, and Nicolas Ballas.
\newblock Revisiting feature prediction for learning visual representations from video.
\newblock \emph{Transactions on Machine Learning Research}, pp.\  1--28, 2024.

\bibitem[Daunhawer et~al.(2023)Daunhawer, Bizeul, Palumbo, Marx, and Vogt]{daunhawer2023identifiabilityresultsmultimodalcontrastive}
Imant Daunhawer, Alice Bizeul, Emanuele Palumbo, Alexander Marx, and Julia~E Vogt.
\newblock Identifiability results for multimodal contrastive learning.
\newblock In \emph{ICLR}, 2023.

\bibitem[Dawid \& LeCun(2024)Dawid and LeCun]{Dawid_2024}
Anna Dawid and Yann LeCun.
\newblock Introduction to latent variable energy-based models: a path toward autonomous machine intelligence.
\newblock \emph{Journal of Statistical Mechanics: Theory and Experiment}, pp.\  1--32, 2024.

\bibitem[He et~al.(2022)He, Chen, Xie, Li, Dollár, and Girshick]{he2021maskedautoencodersscalablevision}
Kaiming He, Xinlei Chen, Saining Xie, Yanghao Li, Piotr Dollár, and Ross Girshick.
\newblock Masked autoencoders are scalable vision learners.
\newblock In \emph{CVPR}, 2022.

\bibitem[Hyvarinen \& Morioka(2017)Hyvarinen and Morioka]{hyvarinen2017nonlinear}
Aapo Hyvarinen and Hiroshi Morioka.
\newblock Nonlinear ica of temporally dependent stationary sources.
\newblock In \emph{AISTATS}, 2017.

\bibitem[Hyv{\"a}rinen \& Pajunen(1999)Hyv{\"a}rinen and Pajunen]{hyvarinen1999nonlinear}
Aapo Hyv{\"a}rinen and Petteri Pajunen.
\newblock Nonlinear independent component analysis: Existence and uniqueness results.
\newblock \emph{Neural Networks}, pp.\  429--439, 1999.

\bibitem[Hyv{\"a}rinen et~al.(2024)Hyv{\"a}rinen, Khemakhem, and Monti]{hyvarinen2024identifiability}
Aapo Hyv{\"a}rinen, Ilyes Khemakhem, and Ricardo Monti.
\newblock Identifiability of latent-variable and structural-equation models: from linear to nonlinear: A. hyv{\"a}rinen et al.
\newblock \emph{Annals of the Institute of Statistical Mathematics}, pp.\  1--33, 2024.

\bibitem[Hyvärinen \& Morioka(2016)Hyvärinen and Morioka]{hyvarinen2016unsupervised}
Aapo Hyvärinen and Hiroshi Morioka.
\newblock Unsupervised feature extraction by time-contrastive learning and nonlinear ica.
\newblock In \emph{NeurIPS}, 2016.

\bibitem[Kingma \& Welling(2014)Kingma and Welling]{kingma2022autoencodingvariationalbayes}
Diederik~P Kingma and Max Welling.
\newblock Auto-encoding variational bayes.
\newblock In \emph{ICLR}, 2014.

\bibitem[Klindt et~al.(2026)Klindt, LeCun, and Balestriero]{klindt2026doeslejepalearnworld}
David Klindt, Yann LeCun, and Randall Balestriero.
\newblock When does lejepa learn a world model?
\newblock \emph{arXiv preprint arXiv:2605.26379}, 2026.

\bibitem[K{\"u}gelgen et~al.(2021)K{\"u}gelgen, Sharma, Gresele, Brendel, Sch{\"o}lkopf, Besserve, and Locatello]{kugelgen2021selfsupervised}
Julius~Von K{\"u}gelgen, Yash Sharma, Luigi Gresele, Wieland Brendel, Bernhard Sch{\"o}lkopf, Michel Besserve, and Francesco Locatello.
\newblock Self-supervised learning with data augmentations provably isolates content from style.
\newblock In \emph{NeurIPS}, 2021.

\bibitem[Locatello et~al.(2019)Locatello, Bauer, Lucic, Rätsch, Gelly, Schölkopf, and Bachem]{locatello2019challengingcommonassumptionsunsupervised}
Francesco Locatello, Stefan Bauer, Mario Lucic, Gunnar Rätsch, Sylvain Gelly, Bernhard Schölkopf, and Olivier Bachem.
\newblock Challenging common assumptions in the unsupervised learning of disentangled representations.
\newblock \emph{arXiv preprint arXiv:1811.12359}, 2019.

\bibitem[Lyu et~al.(2022)Lyu, Fu, Wang, and Lu]{lyu2022understanding}
Qi~Lyu, Xiao Fu, Weiran Wang, and Songtao Lu.
\newblock Understanding latent correlation-based multiview learning and self-supervision: An identifiability perspective.
\newblock In \emph{ICLR}, 2022.

\bibitem[Maes et~al.(2026)Maes, Lidec, Scieur, LeCun, and Balestriero]{maes2026leworldmodelstableendtoendjointembedding}
Lucas Maes, Quentin~Le Lidec, Damien Scieur, Yann LeCun, and Randall Balestriero.
\newblock Leworldmodel: Stable end-to-end joint-embedding predictive architecture from pixels.
\newblock \emph{arXiv preprint arXiv:2603.19312}, 2026.

\bibitem[Moran \& Aragam(2026)Moran and Aragam]{moran2026toward}
Gemma Moran and Bryon Aragam.
\newblock Toward interpretable deep generative models via causal representation learning.
\newblock \emph{Journal of the American Statistical Association}, pp.\  259--275, 2026.

\bibitem[Nair et~al.(2022)Nair, Rajeswaran, Kumar, Finn, and Gupta]{nair2022rm}
Suraj Nair, Aravind Rajeswaran, Vikash Kumar, Chelsea Finn, and Abhinav Gupta.
\newblock R3m: A universal visual representation for robot manipulation.
\newblock In \emph{CoRL}, 2022.

\bibitem[Pang et~al.(2025)Pang, Tang, Li, Tang, Cai, Zhang, Niu, Sugiyama, and Yu]{pang2025learning}
Jing-Cheng Pang, Nan Tang, Kaiyuan Li, Yuting Tang, Xin-Qiang Cai, Zhen-Yu Zhang, Gang Niu, Masashi Sugiyama, and Yang Yu.
\newblock Learning view-invariant world models for visual robotic manipulation.
\newblock In \emph{ICLR}, 2025.

\bibitem[Reizinger et~al.(2026)Reizinger, Mucs{\'a}nyi, Guo, Eysenbach, Sch{\"o}lkopf, and Brendel]{reizinger2026skill}
Patrik Reizinger, B{\'a}lint Mucs{\'a}nyi, Siyuan Guo, Benjamin Eysenbach, Bernhard Sch{\"o}lkopf, and Wieland Brendel.
\newblock Skill learning via policy diversity yields identifiable representations for reinforcement learning.
\newblock In \emph{ICLR}, 2026.

\bibitem[Schölkopf et~al.(2021)Schölkopf, Locatello, Bauer, Ke, Kalchbrenner, Goyal, and Bengio]{scholkopf2021causalrepresentationlearning}
Bernhard Schölkopf, Francesco Locatello, Stefan Bauer, Nan~Rosemary Ke, Nal Kalchbrenner, Anirudh Goyal, and Yoshua Bengio.
\newblock Towards causal representation learning.
\newblock \emph{arXiv preprint arXiv:2102.11107}, 2021.

\bibitem[Sermanet et~al.(2018)Sermanet, Lynch, Chebotar, Hsu, Jang, Schaal, Levine, and Brain]{sermanet2017time}
Pierre Sermanet, Corey Lynch, Yevgen Chebotar, Jasmine Hsu, Eric Jang, Stefan Schaal, Sergey Levine, and Google Brain.
\newblock Time-contrastive networks: Self-supervised learning from video.
\newblock In \emph{ICRA}, 2018.

\bibitem[Shrestha \& Fu(2025)Shrestha and Fu]{shrestha2025contentstyle}
Sagar Shrestha and Xiao Fu.
\newblock Content-style learning from unaligned domains: Identifiability under unknown latent dimensions.
\newblock In \emph{ICLR}, 2025.

\bibitem[Tao et~al.(2025)Tao, Xiang, Shukla, Qin, Hinrichsen, Yuan, Bao, Lin, Liu, Chan, Gao, Li, Mu, Xiao, Gurha, N, Choi, Chen, Huang, Calandra, Chen, Luo, and Su]{tao2025maniskill}
Stone Tao, Fanbo Xiang, Arth Shukla, Yuzhe Qin, Xander Hinrichsen, Xiaodi Yuan, Chen Bao, Xinsong Lin, Yulin Liu, Tse-Kai Chan, Yuan Gao, Xuanlin Li, Tongzhou Mu, Nan Xiao, Arnav Gurha, Viswesh N, Yong~Woo Choi, Yen-Ru Chen, Zhiao Huang, Roberto Calandra, Rui Chen, Shan Luo, and Hao Su.
\newblock Maniskill3: {GPU} parallelized robot simulation and rendering for generalizable embodied {AI}.
\newblock In \emph{7th Robot Learning Workshop}, 2025.

\bibitem[Timilsina et~al.(2024)Timilsina, Shrestha, and Fu]{timilsina2024identifiablesharedcomponentanalysis}
Subash Timilsina, Sagar Shrestha, and Xiao Fu.
\newblock Identifiable shared component analysis of unpaired multimodal mixtures.
\newblock In \emph{NeurIPS}, 2024.

\bibitem[Timilsina et~al.(2026)Timilsina, Nguyen, Shrestha, and Fu]{timilsina2026contentstyle}
Subash Timilsina, Hoang-Son Nguyen, Sagar Shrestha, and Xiao Fu.
\newblock Content-style identification via differential independence.
\newblock In \emph{ICML}, 2026.

\bibitem[Wang et~al.(2025)Wang, Gupta, Zhang, Tonekaboni, Jegelka, Jaakkola, and Uhler]{wang2025an}
Chenyu Wang, Sharut Gupta, Xinyi Zhang, Sana Tonekaboni, Stefanie Jegelka, Tommi Jaakkola, and Caroline Uhler.
\newblock An information criterion for controlled disentanglement of multimodal data.
\newblock In \emph{ICLR}, 2025.

\bibitem[Wu et~al.(2025)Wu, Wang, Raj, and Xia]{wu2025unsupervised}
Yuxuan Wu, Ziyu Wang, Bhiksha Raj, and Gus Xia.
\newblock Unsupervised disentanglement of content and style via variance-invariance constraints.
\newblock In \emph{ICLR}, 2025.

\bibitem[Yan et~al.(2023)Yan, Kong, Gui, Chi, Xing, He, and Zhang]{yan2023counterfactual}
Hanqi Yan, Lingjing Kong, Lin Gui, Yuejie Chi, Eric Xing, Yulan He, and Kun Zhang.
\newblock Counterfactual generation with identifiability guarantees.
\newblock In \emph{NeurIPS}, 2023.

\bibitem[Zhang et~al.(2021)Zhang, McAllister, Calandra, Gal, and Levine]{zhang2021learning}
Amy Zhang, Rowan~Thomas McAllister, Roberto Calandra, Yarin Gal, and Sergey Levine.
\newblock Learning invariant representations for reinforcement learning without reconstruction.
\newblock In \emph{ICLR}, 2021.

\bibitem[Zheng et~al.(2026)Zheng, Li, Fan, Wilson, and Zhang]{zheng2026diverse}
Yujia Zheng, Zijian Li, Shunxing Fan, Andrew~Gordon Wilson, and Kun Zhang.
\newblock Diverse dictionary learning.
\newblock In \emph{ICLR}, 2026.

\bibitem[Zimmermann et~al.(2021)Zimmermann, Sharma, Schneider, Bethge, and Brendel]{zimmermann2022contrastivelearninginvertsdata}
Roland~S. Zimmermann, Yash Sharma, Steffen Schneider, Matthias Bethge, and Wieland Brendel.
\newblock Contrastive learning inverts the data generating process.
\newblock In \emph{ICML}, 2021.

\end{thebibliography}
\bibliographystyle{iclr2027_conference}

\appendix
\section{Related Work}
\textbf{Reconstruction-Free Predictive Learning}
departs from reconstruction-based representation learning by constraining latent representations through predictive relationships between observations rather than by reconstructing the observations themselves~\citep{he2021maskedautoencodersscalablevision,kingma2022autoencodingvariationalbayes}. This perspective also appears in nonlinear ICA, where contrastive and predictive objectives can achieve component-wise identifiability by exploiting temporal variation and additional sources of asymmetry, such as non-Gaussianity~\citep{hyvarinen2016unsupervised,hyvarinen2017nonlinear}. JEPAs~\citep{assran2023selfsupervisedlearningimagesjointembedding,assran2025vjepa2selfsupervisedvideo,bardes2024revisiting,Dawid_2024,maes2026leworldmodelstableendtoendjointembedding} follow this principle by predicting target representations from contextual observations, allowing the learned state to emphasize information that is predictive across observations. LeJEPA~\citep{balestriero2025lejepaprovablescalableselfsupervised,klindt2026doeslejepalearnworld} further connects reconstruction-free prediction to identifiability through predictive alignment and an explicit Gaussianity regularizer. Under stationary Gaussian predictive dynamics and appropriate normalization, its population representation recovers the complete latent state up to a global orthogonal transformation. \modelname builds on this global recovery result but considers a finer identification target: the remaining orthogonal transformation can preserve the complete latent state while still mixing its invariant and variant subspaces.

\textbf{Content-Style Disentanglement}
seeks to decompose observations into a shared component that remains stable across related observations and a variation-specific component. With paired or aligned observations, discriminative learning can identify shared content, and in some settings can further separate shared and private factors, without reconstructing observations~\citep{kugelgen2021selfsupervised,lyu2022understanding,daunhawer2023identifiabilityresultsmultimodalcontrastive}. Recent work also characterizes controlled separation through variance--invariance constraints and information criteria~\citep{wu2025unsupervised,wang2025an}. Other approaches study identification from unpaired domains using distributional variation together with additional statistical or structural conditions~\citep{timilsina2024identifiablesharedcomponentanalysis,shrestha2025contentstyle,yan2023counterfactual}. In particular, differential independence constrains the geometry of content- and style-induced changes, enabling identification without assuming statistical independence between the two factors~\citep{timilsina2026contentstyle}. We adopt the paired-correspondence structure established in this literature and use it to resolve the residual mixing left by predictive identification. Our work thereby recovers both invariant and variant subspaces within the complete predictive representation, up to independent block-wise isometries and without an observation decoder.

\section{Proof}
\label{app:proofs}
\subsection{Proof of Theorem~\ref{th1}}
\label{app:proof_th1}
\begin{theorem*}[Nonlinear Latent Ambiguity]
Let $h=f(z_H,z_L)$. For any smooth nonlinear function $\psi:\mathbb{R}^{d_L}\rightarrow\mathbb{R}^{d_H}$, define
$\tilde z_H=z_H+\psi(z_L),\tilde z_L=z_L$,
and $\tilde f(\tilde z_H,\tilde z_L) = f\bigl(\tilde z_H-\psi(\tilde z_L),\tilde z_L\bigr)$.
Then $\widetilde f(\widetilde z)=f(z)$, while $\frac{\partial\widetilde z_H}{\partial z_L}=J_\psi(z_L)$.
Whenever this derivative is nonzero, the alternative first block depends on the variant factor. 
\end{theorem*}

\begin{proof}
Let the continuous random vector be $z=[z_H^\top,z_L^\top]^\top\in\mathbb R^d$, where $d=d_H+d_L$ and
$h=f(z)$.  We first construct a transformation that induces explicit variant-factor leakage. Choose an arbitrary smooth nonconstant function $\psi:\mathbb R^{d_L}\to\mathbb R^{d_H}$, and define $\widetilde z=V_{\mathrm{leak}}(z)$ by
\begin{equation}
    \widetilde z_H=z_H+\psi(z_L),
    \qquad
    \widetilde z_L=z_L.
\end{equation}
The Jacobian matrix of this transformation is
\begin{equation}
    J_{V_{\mathrm{leak}}}(z)
    =\begin{bmatrix}
       I_{d_H}&J_\psi(z_L)\\
       0&I_{d_L}
     \end{bmatrix}.
\end{equation}
Since $\det J_{V_{\mathrm{leak}}}(z)=1$, the mapping preserves volume. Furthermore,$V_{\mathrm{leak}}^{-1}(\widetilde z_H,\widetilde z_L)=\bigl(\widetilde z_H-\psi(\widetilde z_L),\widetilde z_L\bigr)$, so $V_{\mathrm{leak}}$ is a global smooth diffeomorphism.

Returning to the explicit leakage construction, define the alternative observation-generating map
\begin{equation}
    \widetilde f=f\circ V_{\mathrm{leak}}^{-1}.
\end{equation}
Then
\begin{equation}
    h=f(z)
     =f\bigl(V_{\mathrm{leak}}^{-1}(\widetilde z)\bigr)
     =\widetilde f(\widetilde z).
\end{equation}
Therefore, if the alternative latent is the pushforward
$P_{\widetilde z}=(V_{\mathrm{leak}})_\#P_z$, the pairs
$(f,P_z)$ and $(\widetilde f,P_{\widetilde z})$ induce the same observed marginal distribution $p(h)$.

\paragraph{Conclusion.}
The true latent model $(f,P_z)$ and the alternative latent model
$(\widetilde f,P_{\widetilde z})$ generate the same observed marginal distribution $p(h)$. Consequently, observations of $h$ alone cannot distinguish the true invariant latent factor $z_H$ from the alternative factor $\widetilde z_H=z_H+\psi(z_L)$, which contains nonlinear information from the variant latent factor $z_L$. Accordingly, an encoder $\phi:\mathcal H\to\mathbb R^d$ cannot identify the learned invariant
representation $r_H=[\phi(h)]_H$ with the true $z_H$ from the observation law alone.
\end{proof}

\subsection{Proof of Theorem~\ref{th2}}
\label{app:proof_th2}
\begin{theorem*}[Reconstruction-free Block Identifiability]
Under the predictive model in Equation~\ref{eq:predictive_dynamics} and
Assumption~\ref{assump:sufficient_variation}, let $g^\star$ be a global population minimizer of
$\mathcal L_{\mathrm{pred}}$ over the centered and whitened representation class specified in  Section~\ref{sec:estimation}.
If $g_H^\star(z^{(1)})=g_H^\star(z^{(2)})$ almost surely under $\mathcal D_{\mathrm{inv}}$, then
\begin{equation}
g^\star(z)=Qz
=
\begin{bmatrix}
A_{11} & 0\\
0 & A_{22}
\end{bmatrix}
\begin{bmatrix}
z_H\\
z_L
\end{bmatrix},
\end{equation}
where $A_{11}^\top A_{11}=I_{d_H}$ and
$A_{22}^\top A_{22}=I_{d_L}$. Hence, the invariant and variant latent subspaces are
identified up to independent within-block orthogonal transformations without observation reconstruction.
\end{theorem*}

\begin{proof}
The proof follows the identifiability decomposition induced by the predictive objective and the shared trajectory structure. We first show that the predictive dynamics restrict the optimal representation to a linear metric-preserving embedding through Hermite spectral analysis. Subsequently, the shared invariance removes the cross-block dependency, while orthogonality removes the remaining residual leakage.

\textbf{Step 1. Predictive Dynamics Strictly Isolates the Linear
Eigenspace.}

We analyze the reconstruction-free predictive objective defined on the positive pairs introduced in Section~\ref{sec:background and formulation}. By  the normalization constraint in Equation~\ref{eq:normalization_constraint}, $\operatorname{Cov}(g(z))=I_d$ and $\mathbb E[g(z)]=0$. Therefore, $\mathbb E\|g(z)\|_2^2=d$. The Hermite expansion energy is globally normalized: $\sum_{i=1}^{d} \sum_{|\alpha|\ge1} (c_\alpha^{(i)})^2 =d$. Under this constraint, minimizing the predictive distance is equivalent to
maximizing $\mathbb E[\langle g(z^{(1)}),g(z^{(2)})\rangle]$.

Following the reconstruction-free predictive learning framework
introduced in JEPA-style representation learning~\citep{assran2023selfsupervisedlearningimagesjointembedding,bardes2024revisiting},
and its theoretical analysis in LeJEPA~\citep{klindt2026doeslejepalearnworld} and the predictive model in Equation~\ref{eq:predictive_dynamics}, we models the context-conditioned transition as
\begin{equation}
    z^{(2)}=\rho_X z^{(1)}+\sqrt{1-\rho_X^2}\,\epsilon,
    \qquad \epsilon\sim\mathcal N(0,I_d),
\end{equation} where conditionally on $X$, $\epsilon$ is independent of $z^{(1)}$ and $0<\rho_X<1$. Stationarity means
$z^{(1)}\mid X\sim\mathcal N(0,I_d)$ and hence also $z^{(2)}\mid X\sim\mathcal N(0,I_d)$.

Because the representation is centered by the normalization constraint in Equation~\ref{eq:normalization_constraint}, the degree-0 constant term vanishes ($c_0^{(i)} = 0$). The same constraint gives $\operatorname{Var}(g_i(z))=\sum_{|\alpha|\ge1}\bigl(c_\alpha^{(i)}\bigr)^2=1$.
Grouping by total degree $k = |\alpha|$, let $w_k^{(i)}=\sum_{|\alpha|=k}\bigl(c_\alpha^{(i)}\bigr)^2$, and $\sum_{k\ge1}w_k^{(i)}=1$.

To compute the cross-trajectory correlation exactly, the calculation proceeds in three steps:

\textit{(a) \textbf{The transition attenuates each Hermite degree by $\rho_X^k$ conditionally on $X$.}}
Consider the univariate case for the $j$-th coordinate, $z_j^{(2)}=\rho_Xz_j^{(1)}+\sqrt{1-\rho_X^2}\epsilon_j$. We use the exponential generating function for Hermite polynomials:
$$\sum_{k=0}^\infty \frac{s^k}{k!} He_k(x) = \exp\left(sx - \frac{s^2}{2}\right)$$
evaluating at $z_j^{(2)}$, and taking conditional expectation over
$\epsilon_j$ gives
\begin{equation}
\mathbb E_{\epsilon_j}\left[\exp\!\left(sz_j^{(2)}-\frac{s^2}{2}\right)\middle|z_j^{(1)},X\right]=\exp\!\left(s\rho_Xz_j^{(1)}-\frac{(s\rho_X)^2}{2}\right).
\end{equation}
Matching coefficients gives
\begin{equation}
\mathbb E_{\epsilon_j}\left[He_k\!\left(z_j^{(2)}\right)\mid z_j^{(1)},X\right]=\rho_X^k He_k\!\left(z_j^{(1)}\right).
\end{equation}

\textit{(b) \textbf{Different Hermite components do not interact across views.}}
In the multivariate case, conditional independence across coordinates and iterated expectation give
\begin{equation}
\mathbb E[H_\alpha(z^{(2)})H_\beta(z^{(1)})\mid X]=\rho_X^{|\alpha|}\mathbb E[H_\alpha(z^{(1)})H_\beta(z^{(1)})\mid X]=\rho_X^{|\alpha|}\delta_{\alpha\beta}.
\end{equation}
We use the normalized multivariate Hermite basis
$H_\alpha(z)=\prod_{j=1}^{d} \mathrm{He}_{\alpha_j}(z_j)/\sqrt{\alpha_j!}$.
After averaging over $X$, define
\begin{equation}
    \lambda_k=\mathbb E_X[\rho_X^k],
\end{equation}
so that
\begin{equation}
    \mathbb E[H_\alpha(z^{(2)})H_\beta(z^{(1)})]
    =\lambda_{|\alpha|}\delta_{\alpha\beta}.
\end{equation}

\textit{(c) \textbf{Nonlinearity strictly reduces correlation.}}
The total cross-trajectory correlation of coordinate $i$ is
\begin{equation}
\mathbb E[g_i(z^{(2)})g_i(z^{(1)})]=\sum_{k\ge1}w_k^{(i)}\lambda_k.
\end{equation}
Since $0<\rho_X<1$, we have $\lambda_k<\lambda_1$ for every $k\ge2$. Therefore,
\begin{equation}
\sum_{k\ge1}w_k^{(i)}\lambda_k\le \lambda_1\sum_{k\ge1}w_k^{(i)}=\lambda_1,
\end{equation}
and equality requires $w_k^{(i)}=0$ for all $k\ge2$.  Thus all variance at a global optimum lies in degree one.  Since the degree-1 Hermite basis consists of the linear coordinate functions, $r=g(z)=Qz$.

The normalization constraint in Equation~\ref{eq:normalization_constraint} now gives $I_d=\operatorname{Cov}(g(z))=QQ^\top$. Because $g$ has output dimension $d$, the matrix $Q\in\mathbb R^{d\times d}$ is square;
hence $Q$ is orthogonal and also satisfies $Q^\top Q=I_d$.  Partition it as
\begin{equation}
    Q=\begin{bmatrix}A_{11}&A_{12}\\A_{21}&A_{22}\end{bmatrix},
\end{equation}
where $A_{11}\in\mathbb R^{d_H\times d_H}$, $A_{12}\in\mathbb R^{d_H\times d_L}$, $A_{21}\in\mathbb R^{d_L\times d_H}$, and $A_{22}\in\mathbb R^{d_L\times d_L}$.
The Hermite argument establishes $g(z)=Qz$ almost everywhere under the Gaussian marginal. Since $g$ is continuous and the Gaussian distribution has full support on $\mathbb R^d$, the equality extends to every point in the latent support, including the endpoints occurring under $\mathcal D_{\mathrm{inv}}$.

\vspace{1ex}

\textbf{Step 2. Content Invariance Eliminates the Upper-Right Block($A_{12}=0$).}

Since Step 1 establishes $r=g(z)=Qz$, the learned invariant block is $r_H=g_H(z)=A_{11}z_H+A_{12}z_L$, $\frac{\partial r_H}{\partial z_L}=A_{12}$, and the elimination of cross-block dependence follows the identifiability arguments developed in recent content-style separation theory~\citep{shrestha2025contentstyle}.

We want to establish that the extracted invariant block does not depend on the variant factor, which is equivalent to proving that the partial derivative strictly vanishes:
$$\frac{\partial r_H}{\partial z_L} = A_{12} = \mathbf{0}$$

Because $g^\star$ satisfies
$\mathcal L_{\mathrm{inv}}(g^\star)=0$, we have $g_H^\star(z^{(1)})=g_H^\star(z^{(2)})$
almost surely for $(z^{(1)},z^{(2)})\sim\mathcal D_{\mathrm{inv}}$. Since these pairs satisfy $z_H^{(1)}=z_H^{(2)}$, it follows that $A_{12}\left(z_L^{(1)}-z_L^{(2)}\right)= A_{12}\Delta z_L = 0$ almost surely.

By Assumption~\ref{assump:sufficient_variation},
$\Sigma_{\Delta L}=\mathbb E[\Delta z_L\Delta z_L^\top]\succ0$.  Hence
\begin{equation}
0=\mathbb E\|A_{12}\Delta z_L\|_2^2=\operatorname{Tr}(A_{12}\Sigma_{\Delta L}A_{12}^\top)=\|A_{12}\Sigma_{\Delta L}^{1/2}\|_F^2
\end{equation}
Since $\Sigma_{\Delta L}^{1/2}$ is invertible, $A_{12}=0$.

\vspace{1ex}

\textbf{Step 3.  Metric Preservation Removes the Remaining Cross Block ($A_{21}=0$).}

With $A_{12}=0$, the upper-left block of $QQ^\top=I_d$ is
\begin{equation}
    A_{11}A_{11}^\top=I_{d_H}.
\end{equation}
Thus the square matrix $A_{11}$ is invertible. The upper-right block of $QQ^\top=I_d$ is
\begin{equation}
A_{11}A_{21}^\top+A_{12}A_{22}^\top=A_{11}A_{21}^\top=0.
\end{equation}
Multiplying by $A_{11}^{-1}$ gives $A_{21}=0$. Finally, the lower-right block gives $A_{22}A_{22}^\top=I_{d_L}$.
Because $A_{11}$ and $A_{22}$ are square, we equivalently have
\begin{equation}
A_{11}^\top A_{11}=I_{d_H},\qquad A_{22}^\top A_{22}=I_{d_L}.
\end{equation}
\vspace{1ex}

\paragraph{Conclusion.}
Substituting $A_{12}=0$ and $A_{21}=0$ into $Q$ gives
\begin{equation}
J_g=Q=\begin{bmatrix}A_{11}&0\\0&A_{22}\end{bmatrix}.
\end{equation}
Therefore, the learned representation $r=g(z)$ decomposes as
\begin{equation}
r_H=A_{11}z_H,\qquad r_L=A_{22}z_L,
\end{equation}
where $A_{11}^\top A_{11}=I_{d_H}$ and $A_{22}^\top A_{22}=I_{d_L}$. Hence, $r_H$ recovers the invariant latent
subspace associated with $z_H$, while $r_L$ recovers the variant latent subspace associated with $z_L$, in each case up to an independent orthogonal transformation.
\end{proof}

\subsection{Proof of Theorem~\ref{th3}}
\label{app:proof_th3}
\begin{theorem*}[Gradient Isolation]
Under the block identifiability established in Theorem~\ref{th2}, suppose a differentiable downstream predictor depends solely on the invariant representation $r_H$.
Then the prediction is locally insensitive to perturbations in $z_L$: $\frac{\partial Y}{\partial z_L}=0$. Consequently, local variant perturbations cannot propagate
through the invariant representation.
\end{theorem*}

\begin{proof}
By Theorem~\ref{th2}, the learned representation admits the block-diagonal decomposition
$$r=g(z)=\begin{bmatrix}A_{11}&0\\0&A_{22}\end{bmatrix}\begin{bmatrix}z_H\\z_L\end{bmatrix}$$
which implies$\frac{\partial r_H}{\partial z_L}=0$.
Consider a downstream predictor
$$Y=\Psi(\mathrm{Context},r_H)$$
whose prediction depends only on the invariant representation.
Applying the multivariate chain rule,
$$\frac{\partial Y}{\partial z_L}=\frac{\partial Y}{\partial r_H}\frac{\partial r_H}{\partial z_L}=0$$

\paragraph{Conclusion.}
Because $r_H=A_{11}z_H$ is independent of $z_L$, its Jacobian satisfies $\frac{\partial r_H}{\partial z_L}=0$. Therefore, for any differentiable downstream predictor $Y=\Psi(\mathrm{Context},r_H)$ whose context is held fixed, we have $\frac{\partial Y}{\partial z_L}=0$. Therefore, perturbations confined to the variant latent factor $z_L$ cannot propagate to $Y$ through the learned invariant representation $r_H$.
\end{proof}

\section{Synthetic Evaluation Additional Details}
\label{app:simulation}
\subsection{Synthetic Setting Details}
\label{app:simulation setting}
We consider an $8$-dimensional latent state $z=[z_H^\top,z_L^\top]^\top$, where $z_H,z_L\in\mathbb R^4$ are independently distributed as $\mathcal N(0,I_4)$. Observations are generated as $h=f(z)$ using a four-layer smooth nonlinear coupling map with coupling strength $0.5$. For predictive training, we generate stationary Gaussian OU pairs with correlation $\rho=0.5$. The encoder is trained with the reconstruction-free alignment--SIGReg objective for $10{,}000$ updates using batch size $256$, learning rate $3\times10^{-3}$, and SIGReg weight $10^{-6}$. The encoder is then frozen for all subsequent experiments. No ground-truth latent variables, decoder, or reconstruction objective is used during training. For each trained encoder, we first fit a centered ZCA whitening
transformation using $30{,}000$ unlabeled representations. Let $\bar r$ denote the resulting whitened representation. We then construct $N=30{,}000$ matched invariant pairs and define $\Delta\bar r_i = \bar r_i^{(1)}-\bar r_i^{(2)}$. To match the sufficient-variation condition in
Assumption~\ref{assump:sufficient_variation}, we estimate the uncentered second-moment matrix $\widehat M_{\Delta r} = \frac{1}{N} \sum_{i=1}^{N} \Delta\bar r_i\Delta\bar r_i^\top$. Under the exact block model, invariant correspondences satisfy $\Delta z_H=0$, so the null space of the population difference moment is the invariant subspace. We therefore recover the invariant block from the bottom $d_H$ eigenspace of $\widehat M_{\Delta r}$ and use its orthogonal complement as the $d_L$-dimensional variant subspace. The resulting spectral transformation is applied to the frozen whitened representation. 

\emph{\modelname} uses the correctly matched invariant pairs, \emph{LeJEPA} retains the centered ZCA representation without the spectral rotation, and \emph{Shuffled Pairs} randomly permutes the second member of each pair, destroying the shared-invariant correspondence. We omit No Diversity as a separate numerical baseline because it is operationally identical to LeJEPA in the present implementation. Specifically, setting the paired variant states equal gives $\Delta z_L=0$, so the difference covariance contains no informative direction for spectral recovery. The resulting recovery transformation defaults to the
identity head and produces exactly the LeJEPA representation. This is an
implementation-specific realization of the boundary case in which the full-rank variation assumption fails; the invariant subspace is not uniquely determined by the pair differences.

Evaluation uses independent sets of $15{,}000$ probe samples, $15{,}000$ held-out samples, and $15{,}000$ held-out invariant pairs.  We report Global $R^2$ for linear recovery of the complete latent state, block-level $R^2$ for the invariant and variant subspaces, and the cross-block leakage terms $\lVert A_{12}\rVert_F$ and $\lVert A_{21}\rVert_F$. For comparisons across latent dimensions, we normalize the leakage terms as
$\lVert A_{12}\rVert_F/\sqrt{d_H}$ and $\lVert A_{21}\rVert_F/\sqrt{d_L}$. Block purity summarizes separation between the two recovered blocks, defined as $\operatorname{Purity}(A)=1-\frac{\lVert A_{12}\rVert_F^2+\lVert A_{21}\rVert_F^2}{\lVert A\rVert_F^2}$, where a value of one indicates perfect block separation. Main comparisons are aggregated over five seeds ($0$--$4$); the single-seed scaling experiments are intended only as finite-sample diagnostics.

\subsection{Loss Ablations}
\label{app:simulation-ablation}
We isolate the three mechanisms associated with the identifiability pipeline. The \emph{Full} model uses predictive alignment and SIGReg during encoder training, centered ZCA normalization, and the matched-pair spectral head. For \emph{w/o Predictive}, the encoder is trained using SIGReg alone, with the OU-pair alignment term removed. For \emph{w/o Invariance}, we retain the full encoder and ZCA normalization but remove the matched-pair spectral rotation. For \emph{w/o Variance}, we remove both SIGReg during encoder training and ZCA normalization during head fitting, retaining only centering, predictive alignment, and matched-pair recovery, testing the variance-preserving feasible geometry implemented by SIGReg and ZCA. Because this ablation removes two coupled components, it evaluates
their combined contribution to variance-preserving geometry.

As shown in Fig.~\ref{fig:simulation-loss-ablation}, removing predictive alignment decreases Global $R^2$. The invariant and variant scores remain relatively high, but they are no longer meaningful evidence of Theorem~\ref{th2} because the preceding global-identification condition has failed. Without matched-pair recovery the complementary failure mode is produced: Global $R^2$ remains
$0.9977\pm0.0016$, whereas Invariant and Variant $R^2$ fall, and the leakage terms increase to $0.3320\pm0.2152$ and $0.3300\pm0.2149$. Removing the variance preservation leaves the mean block scores near those of the full model but increases the covariance error and increases the variability of $A_{21}$. Thus, in this synthetic setting, predictive and invariance information determine the two recovery stages, while variance preservation mainly improves
normalization and finite-sample stability.

\subsection{Scaling Diagnostics}
\label{app:simulation-scaling}
We conduct three scaling studies by varying the latent dimension, the predictive training budget, and the sample budget used for whitening and spectral-head fitting. In each of the scaling figures, panels (a)--(c) correspond to latent-dimension scaling, encoder training-sample scaling, and head-fitting sample scaling, respectively. Every panel compares Ours, LeJEPA, and Shuffled Pairs (except for Global $R^2$, which remains consistent because of the experiment setting). All curves use seed $0$ to diagnose trends. Together, these diagnostics characterize how predictive training and finite-sample moment estimation approach the population recovery conditions as their respective data budgets increase.

\paragraph{Latent-dimension scaling.}
We vary the total latent dimension over $d\in\{4,8,16,32,64\}, d_H=d_L=d/2$, and increase the representation dimension accordingly while keeping the remaining settings fixed. Panel (a) of Figs.~\ref{fig:scaling_global_r2}--\ref{fig:scaling_normalized_a21}
shows that Global $R^2$ remains high and exceeds $0.9999$ for $d\geq16$. For our method, both block-level $R^2$ values remain above $0.994$ and normalized $\lVert A_{12}\rVert_F$ remains below $1.5\times10^{-3}$. Normalized $\lVert A_{21}\rVert_F$, however, increases from $0.0053$ at $d=4$ to $0.0325$ at $d=64$. This asymmetry indicates that increasing dimension has little effect on global recovery but makes estimation of the complementary block more sensitive.

\paragraph{Encoder training-sample scaling.}
We vary the number of online OU training pairs over $\{2.56\times10^5,\,6.40\times10^5,\,1.28\times10^6,\,2.56\times10^6,\,5.12\times10^6\}$, while keeping the latent dimension, observation map, optimization configuration, and downstream recovery procedure fixed. As shown in
panel (b) of Figs.~\ref{fig:scaling_global_r2}--\ref{fig:scaling_normalized_a21}, Global $R^2$ increases with the training budget, with a clear transition around $1.28\times10^6$ pairs and limited improvement beyond $2.56\times10^6$. The block-level leakage decreases in parallel, indicating that insufficient predictive training primarily manifests as imperfect recovery of the global representation.

\paragraph{Head-fitting sample scaling.}
With the predictive encoder fixed, we vary the common sample budget for ZCA whitening and matched-pair recovery over $\{250,500,1000,3000,10000,30000\}$. Panel (c) of
Figs.~\ref{fig:scaling_global_r2}--\ref{fig:scaling_normalized_a21} shows a strong reduction in normalized $\lVert A_{21}\rVert_F$, from $0.156$ with $250$ samples to $0.012$ with $30000$ samples, while normalized $\lVert A_{12}\rVert_F$ remains below $0.005$ throughout. The Global and block-level $R^2$ curves remain comparatively stable. This performance indicates that the matched-pair covariance already identifies the invariant nullspace reliably, while more samples are primarily needed to estimate whitening and the complementary block accurately.

\begin{figure}[t]
  \centering
  \includegraphics[width=\linewidth]{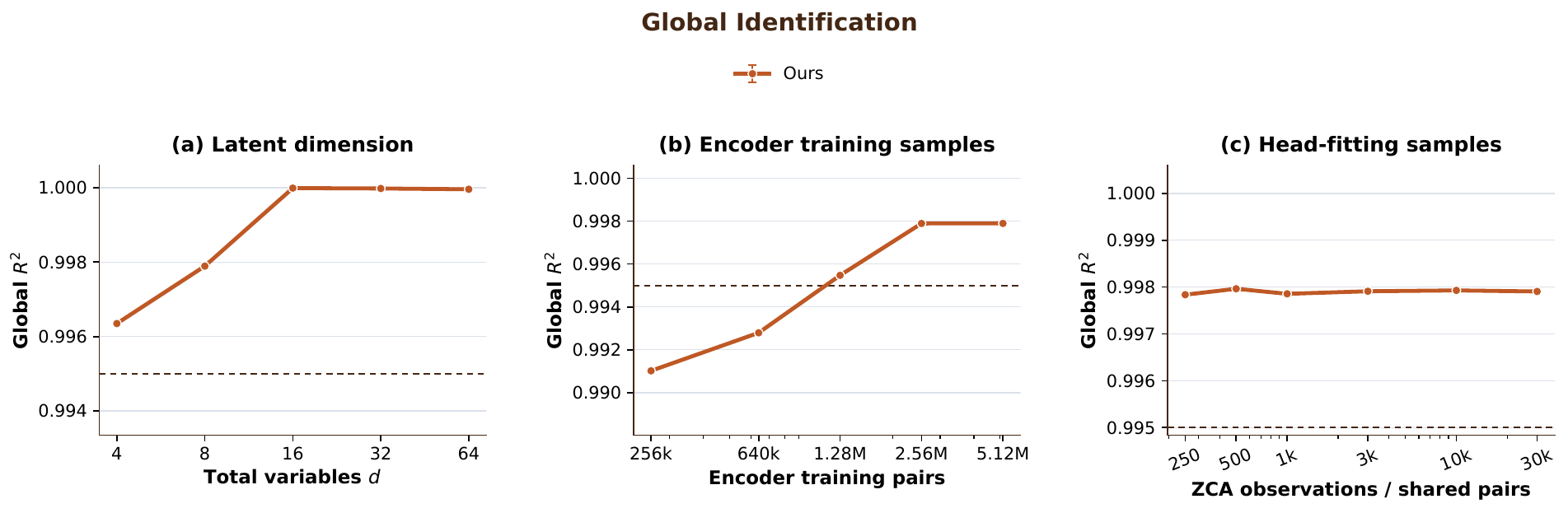}
  \caption{\textbf{Global linear identification under three scaling regimes.}
Higher Global $R^2$ indicates more accurate linear identification of the complete latent representation.}
  \label{fig:scaling_global_r2}
\end{figure}
\begin{figure}[t]
  \centering
  \includegraphics[width=\linewidth]{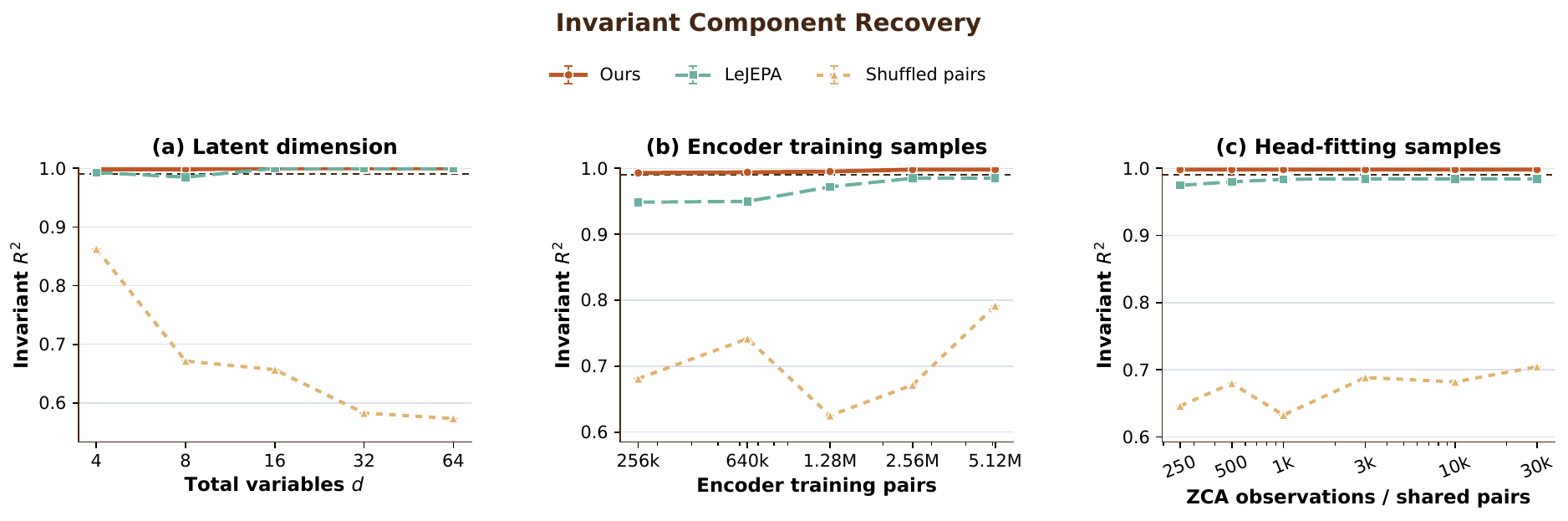}
  \caption{\textbf{Recovery of the invariant latent block under three scaling regimes.} Invariant $R^2$ measures the linear recoverability of the ground-truth invariant factors from the learned invariant representation.}
\label{fig:scaling_invariant_r2}
\vspace{-10pt}
\end{figure}
\begin{figure}[t]
  \centering
  \includegraphics[width=\linewidth]{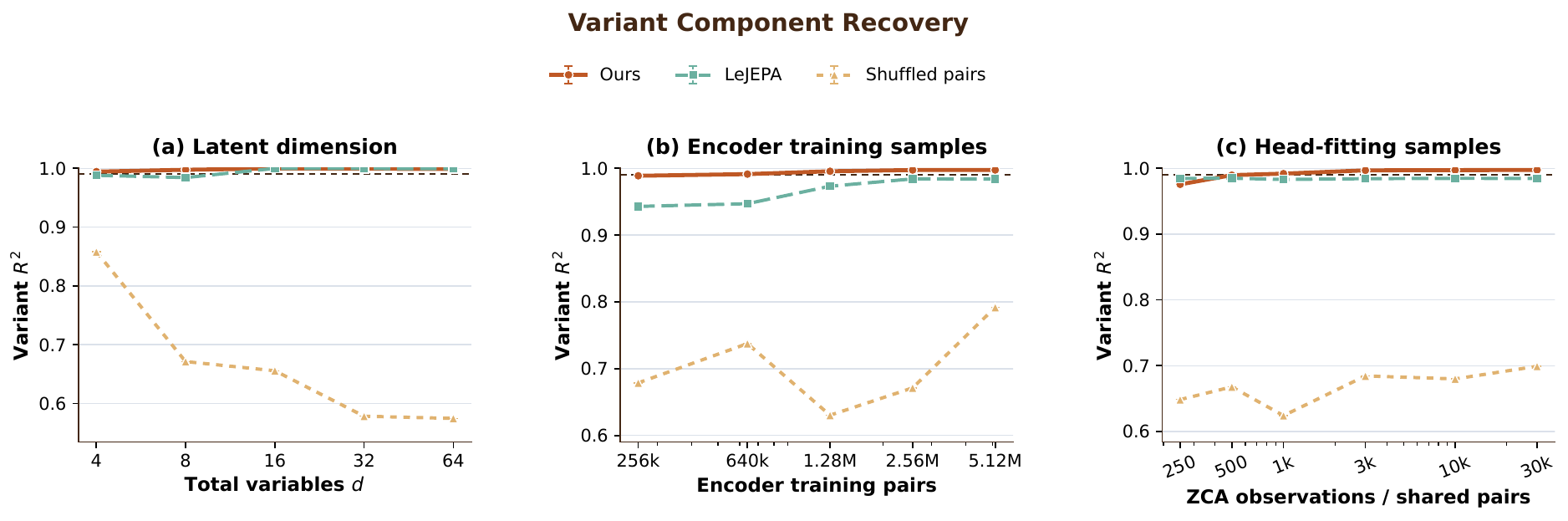}
  \caption{\textbf{Recovery of the variant latent block under three scaling regimes.} Variant $R^2$ measures the linear recoverability of the ground-truth variant factors from the learned variant representation.}
\label{fig:scaling_variant_r2}
\vspace{-15pt}
\end{figure}
\begin{figure}[t]
  \centering
  \includegraphics[width=\linewidth]{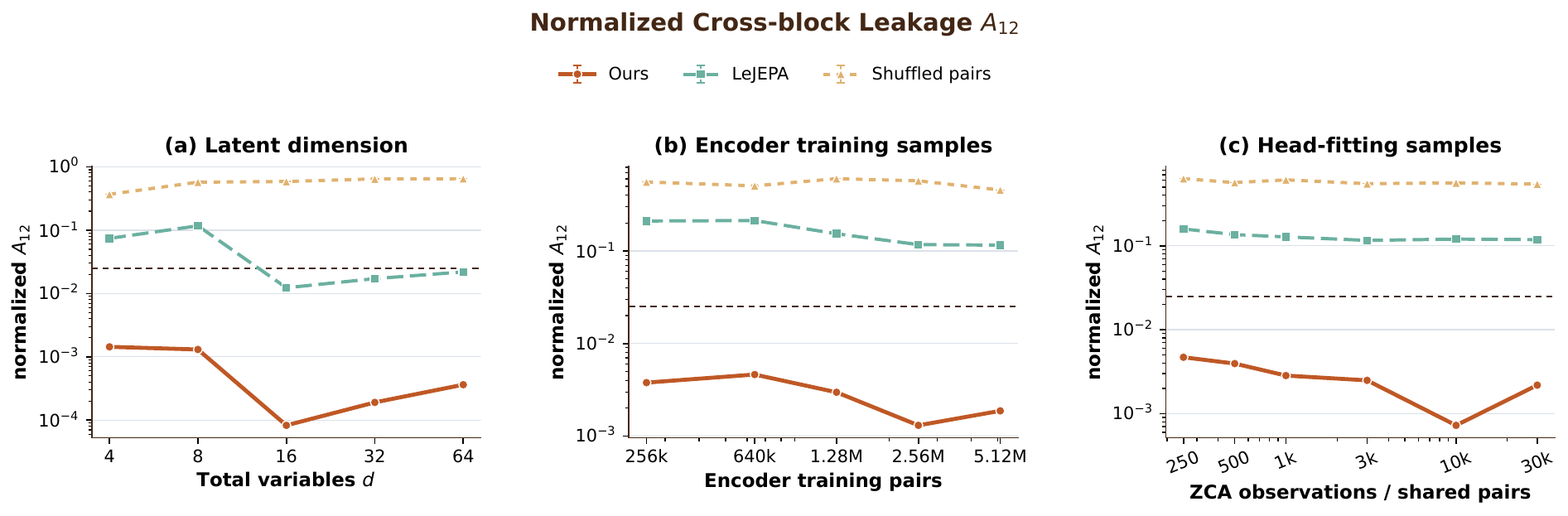}
  \caption{
\textbf{Normalized cross-block leakage $A_{12}$ under three scaling regimes.} The normalized $A_{12}$ score quantifies leakage from the variant latent block into the recovered invariant block and is defined as
$\lVert A_{12}\rVert_F/\sqrt{d_H}$.}
\label{fig:scaling_normalized_a12}
\vspace{-10pt}
\end{figure}
\begin{figure}[t]
  \centering
  \includegraphics[width=\linewidth]{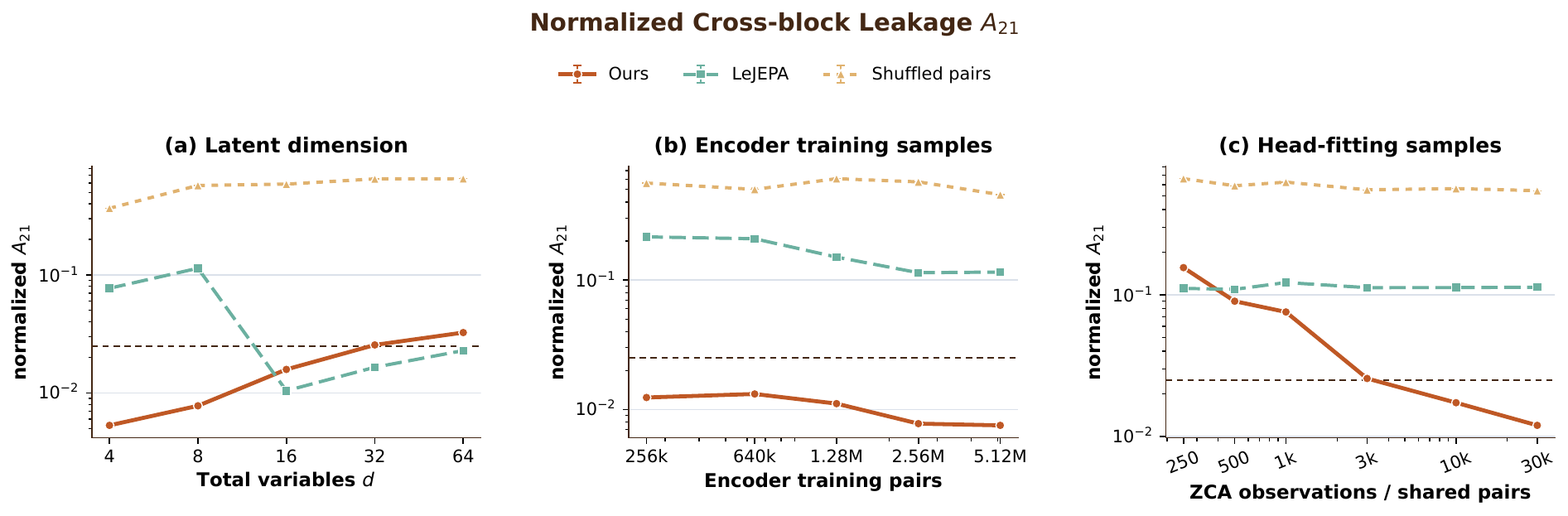}
  \caption{\textbf{Normalized cross-block leakage $A_{21}$ under three scaling regimes.} The normalized $A_{21}$ score quantifies leakage from the invariant latent
block into the recovered variant block and is defined as $\lVert A_{21}\rVert_F/\sqrt{d_L}$.}
\label{fig:scaling_normalized_a21}
\vspace{-10pt}
\end{figure}

\subsection{Efficiency at Matched Recovery Quality}
\paragraph{Efficiency Protocol.}
We compare the end-to-end efficiency of our reconstruction-free implementation with V3~\citep{wu2025unsupervised}, a content-style method based on variance-invariance constraints. Both implementations use the same latent dimension, nonlinear observation map, five random seeds and 256 atomic observations per optimization step. Each method receives its native grouping structure: \modelname uses matched pairs with shared content and varying style, whereas V3 uses groups with shared style and varying content. The two protocols share the same marginal latent distribution and total observation budget. We evaluate the content and style branches of V3 against $z_H$ and $z_L$, respectively. Runtime was measured using a monotonic wall-clock timer. For SplitJEPA, it includes data generation, predictive training, whitening, and matched-pair block recovery. For V3, it includes data generation, encoder-decoder training, vector quantization, and all variance-invariance losses. Model construction and held-out evaluation are excluded for both methods, and no training iterations are discarded as warm-up.
\vspace{-5pt}
\begin{wrapfigure}{r}{0.45\textwidth}
    \centering
    \includegraphics[width=\linewidth]{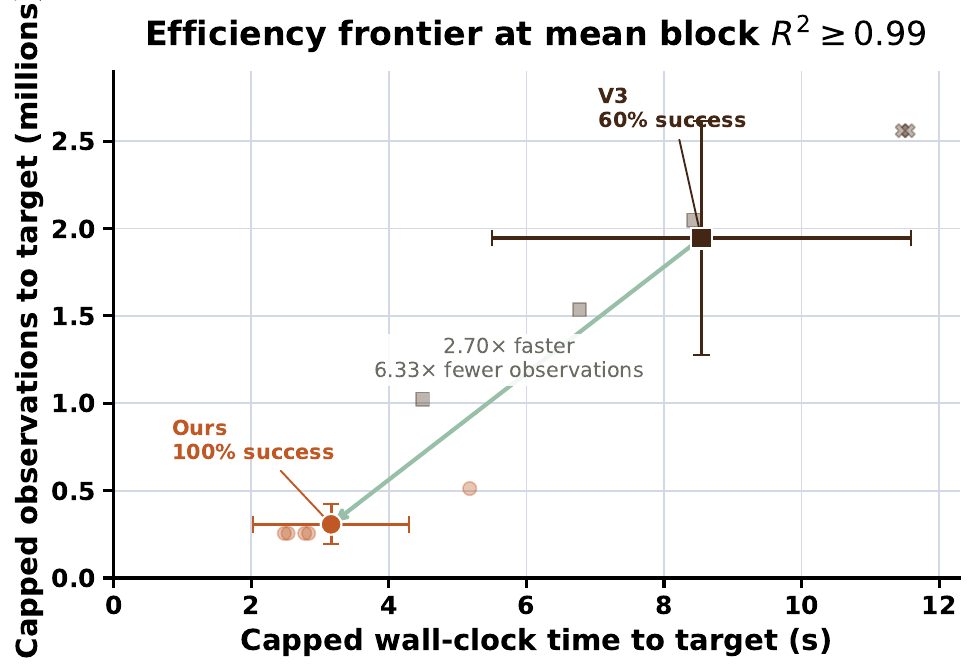}
    \caption{Simulation efficiency comparison.}
    \label{fig:simulation_efficiency}
    \vspace{-45pt}
\end{wrapfigure}

We evaluate both methods at 5\%, 10\%, 20\%, 40\%, 60\%, 80\%, and 100\% of a $10{,}000$-step training budget. The common recovery target is
$\frac{1}{2}(R^2_{\mathrm{Invariant}}+R^2_{\mathrm{Variant}})\geq0.99$.
Every run completes the full training budget. We report the cumulative runtime and observation count at the first evaluated checkpoint that reaches the target. Runs that never reach it are retained and assigned their full-budget runtime and observation count, yielding capped time-to-target and observation-cost metrics. Checkpoint evaluation time is excluded for both methods.

\paragraph{Efficiency Results.}
As shown in Fig.~\ref{fig:simulation_efficiency}, \modelname reaches the recovery target in all five runs, whereas V3 reaches it in three. The capped time-to-target is $3.16\pm1.14$ seconds for \modelname and $8.54\pm3.05$ seconds for V3, corresponding to a $2.70\times$ lower capped runtime for SplitJEPA. The difference in observation cost is larger: \modelname uses $0.307\pm0.114$ million observations under the same capped protocol, compared with $1.946\pm0.668$ million for V3, a $6.33\times$ reduction. The implementations contain 64 and 272 trainable parameters, respectively. The result demonstrates that \modelname achieves a higher target attainment rate with lower capped runtime, lower observation cost, and fewer trainable parameters.

\subsection{Additional Robustness}
\label{app:simulation-robustness}
\paragraph{Full-coordinate Nonlinear MLP.}
To test the block-recovery mechanism beyond the default coupling structure, we concatenate the latent blocks, apply a dense orthogonal affine transformation that mixes the two blocks, and then apply the elementwise map $\eta(u)=u+\alpha\tanh(u)$ for $\alpha\ge0$. Because $\eta'(u)=1+\alpha\operatorname{sech}^2(u)\ge1$, the elementwise map is smooth and injective. We use the same reconstruction-free training and block-recovery pipeline as in the main synthetic experiment. After training, the encoder is frozen and all block controls are applied to the same frozen encoder for each seed.
Table~\ref{tab:simulation-mlp} reports the five-seed results. Our recovery achieves the best invariant and variant $R^2$ values, with $\lVert A_{12}\rVert_F=0.0056\pm0.0007$,
$\lVert A_{21}\rVert_F=0.0187\pm0.0037$. In contrast, the LeJEPA condition retains substantially larger cross-block leakage despite using the same frozen encoder. Shuffled Pairs further degrades both block alignment and purity. These results show that the proposed representation-only recovery remains effective under joint nonlinear mixing, provided that the preceding predictive encoder reaches a sufficiently informative global representation.

\begin{table}[t]
  \centering
  \small
  \setlength{\tabcolsep}{2.8pt}
  \caption{
    \textbf{Reconstruction-free block recovery under joint nonlinear MLP mixing} (mean $\pm$ standard deviation over five seeds).
    All methods reuse the same frozen MLP encoder and differ only in the information available during spectral recovery. Global $R^2$ is approximately $0.977$ for all methods.
  }
  \label{tab:simulation-mlp}
  \begin{tabular}{lccccc}
    \toprule
    Method
      & Invariant $R^2$
      & Variant $R^2$
      & $\lVert A_{12}\rVert_F$
      & $\lVert A_{21}\rVert_F$
      & Block Purity\\
    \midrule
    LeJEPA
      & $0.6778\pm0.0360$
      & $0.6769\pm0.0387$
      & $1.0940\pm0.0696$
      & $1.0910\pm0.0710$
      & $0.6936\pm0.0393$\\
    Shuffled Pairs
      & $0.6598\pm0.0555$
      & $0.6606\pm0.0567$
      & $1.1243\pm0.0997$
      & $1.1220\pm0.1013$
      & $0.6750\pm0.0581$\\
    \textbf{\modelname (Ours)}
      & $\mathbf{0.9770\pm0.0003}$
      & $\mathbf{0.9769\pm0.0002}$
      & $\mathbf{0.0056\pm0.0007}$
      & $\mathbf{0.0187\pm0.0037}$
      & $\mathbf{0.99995\pm0.00002}$\\
    \bottomrule
  \end{tabular}
\end{table}

\paragraph{Coupling-strength sensitivity.}
We further vary the nonlinear coupling strength in the default synthetic configuration while keeping the training and recovery procedure fixed. We consider coupling strengths $0.5$, $1.0$ and $2.0$. Table~\ref{tab:simulation-sensitivity} reports the resulting global identification and block-structure measures. Increasing the coupling strength reduces Global $R^2$ and increases the orthogonality error, while the normalized cross-block terms become $0.0572$ and $0.0210$. At strength $2.0$, Global $R^2$ further decreases to $0.8765$, with orthogonality error and cross-block terms further increasing. Theorem~\ref{th2} conditions on the predictive encoder attaining the
globally identifiable linear solution class. This experiment shows how departures from that geometry under finite optimization appear jointly as increased orthogonality error and cross-block leakage.

\begin{table}[t]
  \centering
  \small
  \setlength{\tabcolsep}{5pt}
  \caption{\textbf{Sensitivity to nonlinear coupling strength.} Seed-$0$ diagnostic under the default synthetic configuration. Increasing the coupling strength progressively degrades the global identification stage.}
  \label{tab:simulation-sensitivity}
  \begin{tabular}{ccccc}
    \toprule
    Coupling Strength
      & Global $R^2$
      & $\lVert Q^\top Q-I\rVert_F/\sqrt{d}$
      & $\lVert A_{12}\rVert_F$
      & $\lVert A_{21}\rVert_F$\\
    \midrule
    $0.5$
      & $\mathbf{0.9979}$
      & $\mathbf{0.0199}$
      & $\mathbf{0.0026}$
      & $\mathbf{0.0155}$\\
    $1.0$
      & $0.9565$
      & $0.0586$
      & $0.0572$
      & $0.0210$\\
    $2.0$
      & $0.8765$
      & $0.1507$
      & $0.1393$
      & $0.0289$\\
    \bottomrule
  \end{tabular}
  \vspace{-10pt}
\end{table}

\section{Robotic Manipulation Experiments Additional Details}
\label{app:realworld}
\subsection{Robotic Manipulation Setting Details}
\label{app:realworld setting}
\paragraph{Data Processing.} We evaluate PickCube-v1\footnote{\url{https://maniskill.readthedocs.io/en/latest/tasks/table_top_gripper/index.html\#pickcube-v1}} and PushCube-v1\footnote{\url{https://maniskill.readthedocs.io/en/latest/tasks/table_top_gripper/index.html\#pushcube-v1}} in ManiSkill3. PickCube requires the robot to grasp a cube, transport it in 3 dimensions, and place it at a target pose, whereas PushCube requires the robot to establish contact and move the cube into a planar goal region. The visual input is a $128\times128$ RGB image from the third-person base camera. At each control step, the policy receives the four most recent images and predicts an 8-dimensional action. The same manipulation dynamics, object initialization distribution, action space, controller, and episode horizon are used for ID and OOD evaluation; only the specified visual nuisance is changed. For each task, the offline dataset contains 800 successful trajectories organized into 100 matched groups of eight. Trajectories within a group share object identity, initial object pose, and goal pose, while differing in robot initialization, execution trajectory, approach or contact direction, and timing. We use two forms of correspondence. Cross-view pairs render the same physical state from the fixed base and side cameras and therefore provide state-matched visual variation. Cross-trajectory pairs are weaker correspondences: frames are approximately aligned by normalized execution progress and preserve group-level task content.  These pairs broaden the coverage of trajectory variation while relaxing exact state alignment, allowing the diversity ablation to examine the trade-off between variation coverage and correspondence precision. We reserve $20\%$ of the trajectories within each group for validation.

\paragraph{Representation Learning.} The robotic implementation follows the same identification decomposition as the synthetic study: prediction in latent space preserves the evolving state, normalization regularizes the geometry of the complete representation, and matched correspondences determine the invariant block. The controlled synthetic setting estimates the final block transform explicitly through ZCA whitening and eigendecomposition. In the visual robotic setting, we optimize counterparts of these components jointly. Both tasks use an ImageNet-initialized ResNet-50 encoder operating on
$128\times128$ RGB observations. The encoder produces a partitioned representation $r=[r_H;r_L]$, with $d_H=d_L=64$ for PickCube and $d_H=d_L=128$ for PushCube. For robotic representation learning, an action-conditioned predictor maps $(r_t,a_t)$ to a stop-gradient target $r_{t+1}$, while the invariance loss is applied only to $r_H$. Variance and covariance regularization are used to prevent representation collapse. Each optimization step combines a batch of 32 adjacent temporal transitions with an independently shuffled batch of 32 matched pairs. Cross-trajectory and cross-view pairs are pooled for training, so their relative frequency follows the number of available pairs. Random pairs with different matched-group identities are excluded from the main training objective and are used only as diagnostic
controls and in the shuffled-correspondence ablation. Held-out cross-trajectory and cross-view pairs are used for block validation. We optimize with AdamW using an initial learning rate of $3\times10^{-4}$, weight decay $10^{-4}$, cosine learning-rate decay, and gradient clipping at norm $10$. PickCube is trained for 200 epochs and PushCube for 100 epochs. Ours-$r_H$ exposes only the learned invariant block to the downstream policy, whereas Ours-$r_H+r_L$ exposes the complete partitioned representation.

\paragraph{Downstream Policies and Evaluation.} All downstream policies are trained by behavior cloning on ID demonstrations. The policy head is a two-layer MLP with hidden dimensions $512$ and $256$. Raw learns directly from RGB observations, R3M uses a frozen pretrained R3M visual feature, and LeJEPA uses a frozen reconstruction-free representation without the proposed block decomposition. These method-specific visual frontends are retained throughout the comparison, and the demonstration budget, temporal context, action target, and evaluation protocol are held fixed. Each trained policy is evaluated on 100 closed-loop episodes under one ID and three OOD visual conditions. We report SR@10cm as the fraction of episodes in which the object reaches within $10$ cm of the goal at any time, using the three-dimensional object-goal distance for PickCube and the planar distance for PushCube. GRASP is the fraction of PickCube episodes in which the environment reports the object as grasped at least once, while CONTACT is the fraction of PushCube episodes in which the TCP comes within $4.5$ cm of the object at least once, serving as a consistent proxy for contact. Camera OOD rotates the base-camera azimuth by $10^\circ$, and the reported value averages the paired $-10^\circ$ and $+10^\circ$ conditions. Lighting OOD scales the intensity of the official scene lighting to $0.6$, while Visual OOD combines the camera and lighting interventions. The same physical reset seeds are shared across methods within each task, and no policy is fine-tuned or adapted using OOD observations. Unless stated otherwise, results are the mean $\pm$ standard deviation over three downstream training seeds.

\subsection{Robustness Experiments}
\label{app:realworld-robustness}
\paragraph{Sample Efficiency.} 
\begin{figure*}[t]
    \centering
    \begin{minipage}[t]{0.49\textwidth}
        \centering
        \textbf{PushCube}
        \vspace{1mm}
        
        \includegraphics[width=\linewidth]{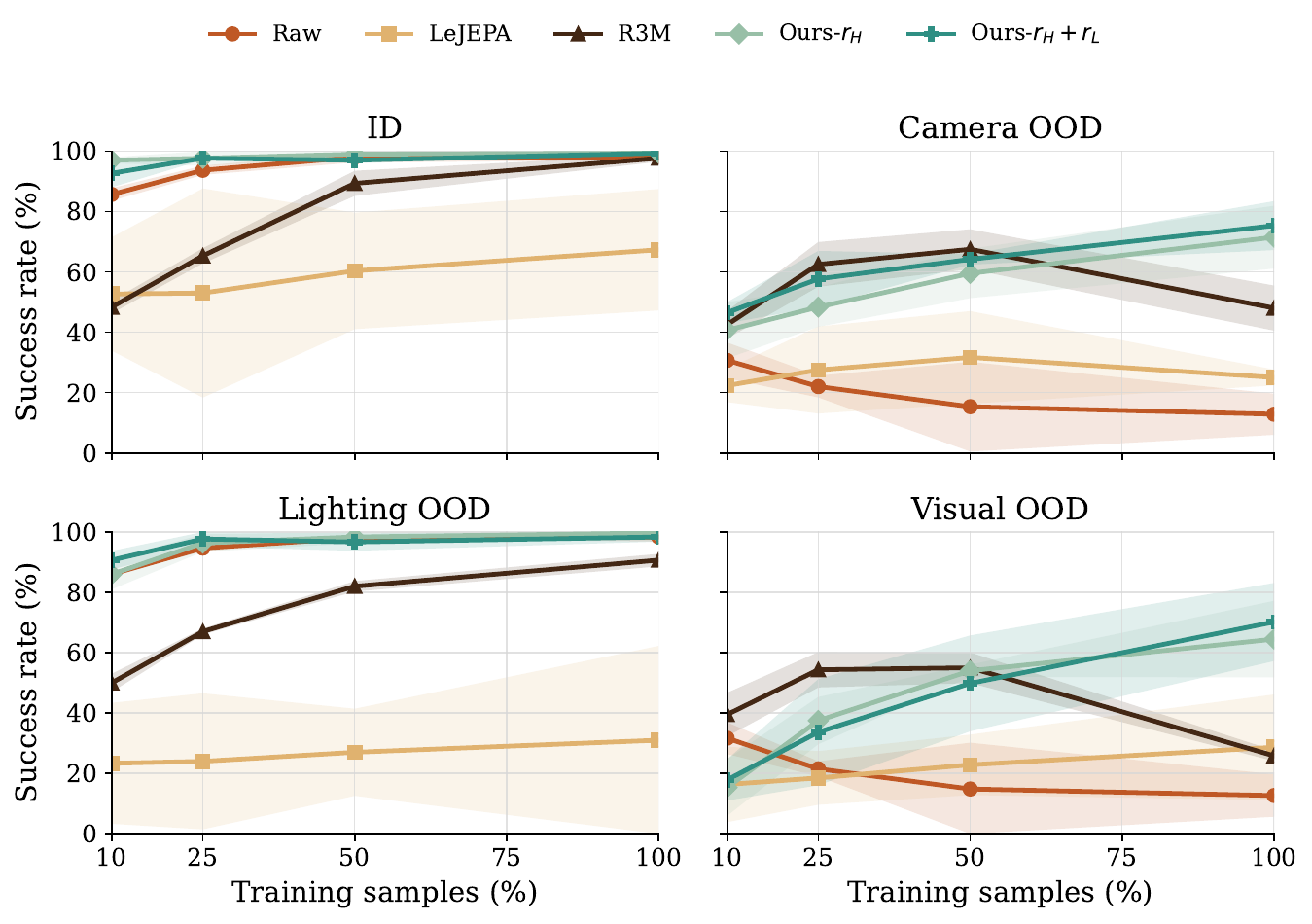}
    \end{minipage}
    \hfill
    \begin{minipage}[t]{0.49\textwidth}
        \centering
        \textbf{PickCube}
        \vspace{1mm}
        
        \includegraphics[width=\linewidth]{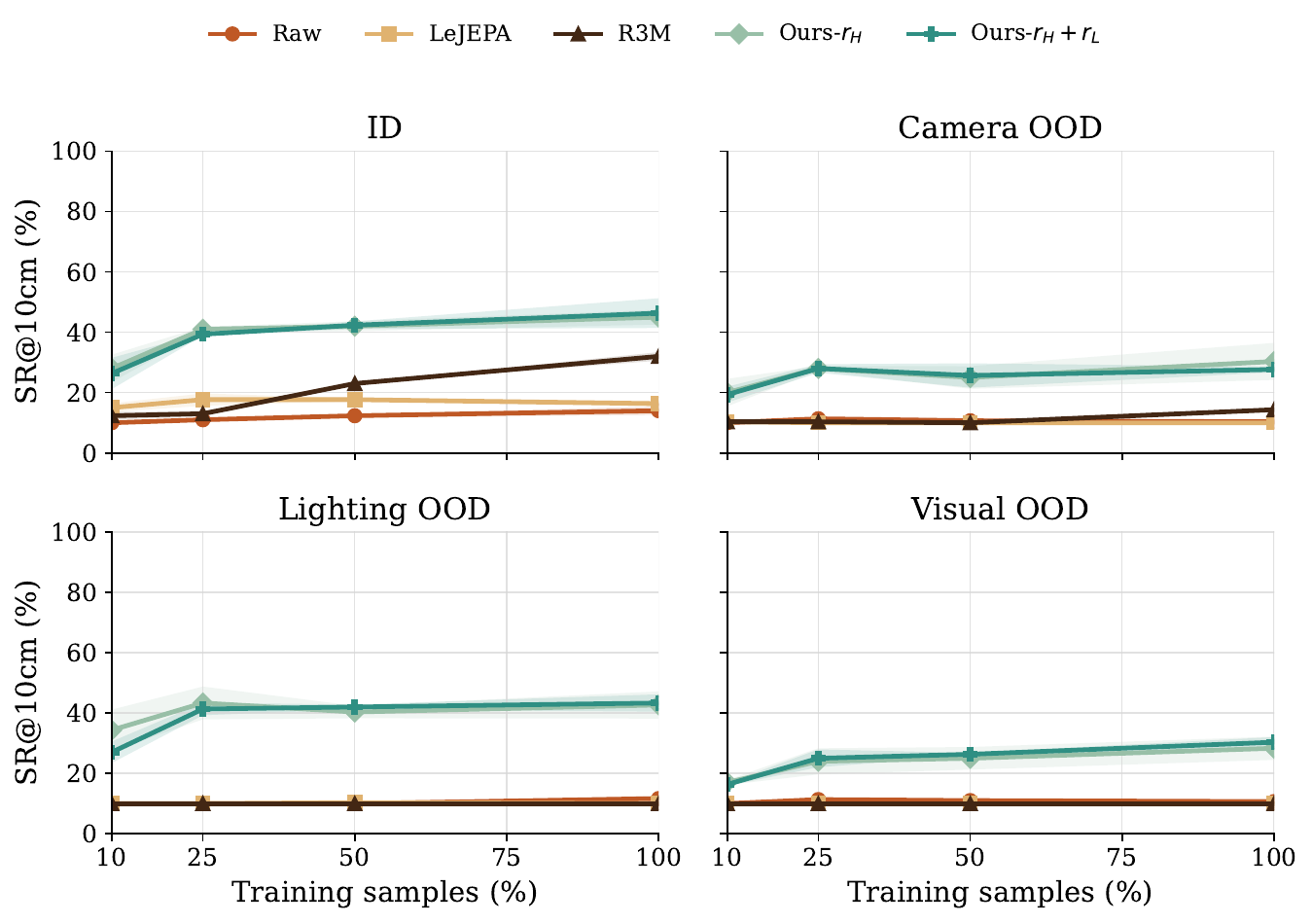}
    \end{minipage}
    
    \caption{
        \textbf{Sample efficiency of learned representations.}
        Performance of different representations under varying amounts of
        training data on PushCube and PickCube, evaluated under
        ID and OOD settings.
    }
    \label{fig:sample_efficiency}
    \vspace{-10pt}
\end{figure*}

\begin{table}[t]
\centering
\caption{
\textbf{Performance retention under representation interventions}}
\label{tab:intervention_retention}
\small
\setlength{\tabcolsep}{4.5pt}
\resizebox{\textwidth}{!}{
\begin{tabular}{lccc|ccc}
\toprule
& \multicolumn{3}{c|}{\textbf{PushCube}}
& \multicolumn{3}{c}{\textbf{PickCube}} \\
\cmidrule(lr){2-4}
\cmidrule(lr){5-7}
\textbf{Intervention}
& $r_H+r_L$ 
& $r_H$ 
& $r_H$ advantage
& $r_H+r_L$ 
& $r_H$ 
& $r_H$ advantage \\
\midrule

Zero $r_L$
& $93.7 \pm 3.4\%$
& $99.9 \pm 0.2\%$
& $+6.1$ pp
& $61.4 \pm 5.1\%$
& $99.7 \pm 0.5\%$
& $+38.4$ pp \\

Noise $0.5\sigma$
& $99.0 \pm 1.0\%$
& $100.0 \pm 0.0\%$
& $+1.0$ pp
& $87.9 \pm 0.8\%$
& $100.0 \pm 0.0\%$
& $+12.1$ pp \\

Noise $1.0\sigma$
& $90.2 \pm 0.5\%$
& $99.9 \pm 0.2\%$
& $+9.7$ pp
& $62.2 \pm 4.6\%$
& $100.0 \pm 0.0\%$
& $+37.8$ pp \\

Noise $2.0\sigma$
& $44.1 \pm 4.0\%$
& $99.9 \pm 0.2\%$
& $+55.8$ pp
& $29.6 \pm 0.8\%$
& $100.0 \pm 0.0\%$
& $+70.4$ pp \\

\bottomrule
\end{tabular}
}
\end{table}

We examine whether the advantage of the learned representations persists when the amount of downstream supervision is reduced. We vary the fraction of behavior-cloning demonstrations from 10\% to 100\%.  As shown in Figure~\ref{fig:sample_efficiency}, the proposed representations remain effective across training budgets with the clearest advantage appearing under visual distribution shifts. On PushCube, both Ours-$r_H$ and Ours-$r_H+r_L$ already achieve above 90\% ID SR@10cm with only 10\% of the demonstrations. Their Camera and Visual OOD performance further improves as the training budget increases. In contrast, the baseline representations either require substantially more demonstrations to approach the ID performance of our method or remain sensitive to the visual shifts. PickCube exhibits a stronger dependence on the amount of policy-training data because the task requires grasping, transporting, and accurately approaching the goal. The ID performance of baseline methods increases with data volume, but under OOD conditions most baselines do not improve with more data, instead reaching a low-performance plateau under these visual shifts. Nevertheless, our two representations consistently outperform the baselines across the evaluated budgets. Neither of our representation inputs dominates uniformly: Ours-$r_H$ is preferable in several low-data or camera-shift settings, whereas Ours-$r_H+r_L$ becomes competitive or stronger when additional execution-specific information is useful. Overall, the results show that the recovered blocks remain usable under limited downstream supervision and that the robustness improvement is not restricted to the full-data condition.

\paragraph{Block Validation.}
\begin{figure}[t]
    \centering

    \begin{minipage}[t]{0.48\linewidth}
        \centering
        \textbf{PickCube}
        \includegraphics[width=\linewidth]{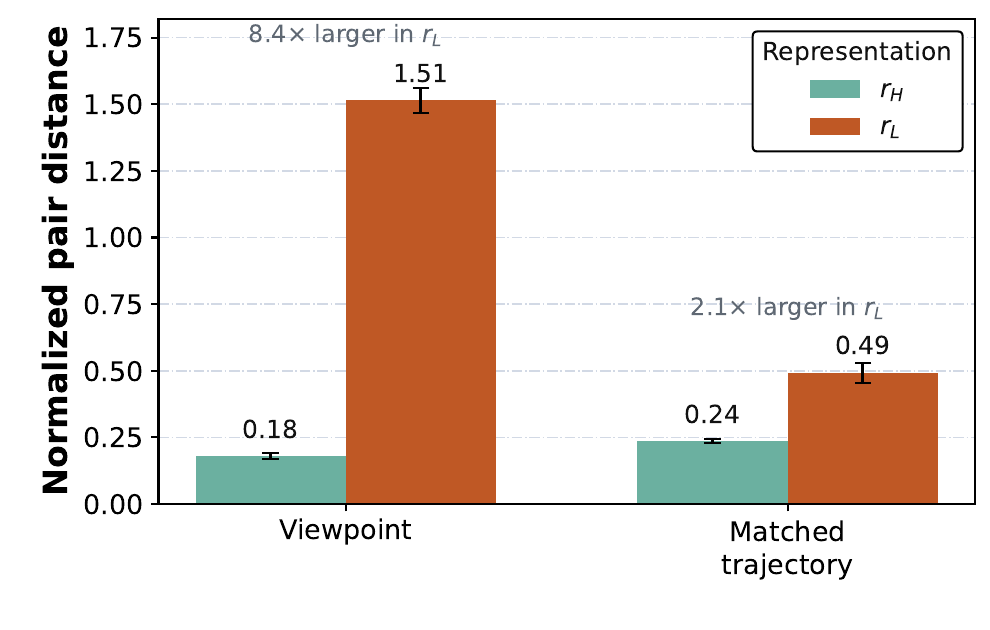}
    \end{minipage}
    \hfill
    \begin{minipage}[t]{0.48\linewidth}
        \centering
        \textbf{PushCube}

        \includegraphics[width=\linewidth]{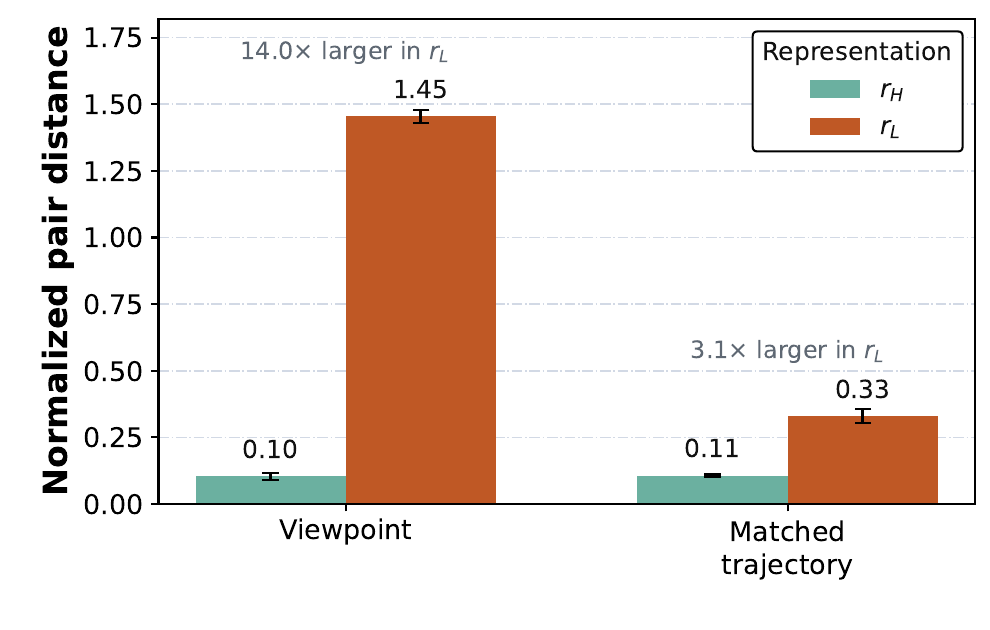}
    \end{minipage}
    \vspace{-10pt}
    \caption{
        \textbf{Empirical validation of invariant-variant block specialization.}
        Representation responses to matched viewpoint and trajectory variations
        on PickCube and PushCube. The learned invariant block $r_H$ changes
        substantially less than the variant block $r_L$ under the same
        variations.}
    \label{fig:block_validation}
    \vspace{-10pt}
\end{figure}
To verify whether our $r_H$ and $r_L$ exhibit the intended differential sensitivity, we measure the representation distance separately in $r_H$ and $r_L$ normalized by the corresponding distance $D_B(x,x')=\frac{1}{\sqrt{d_B}}\left\|\frac{r_B(x)-r_B(x')}{\sigma_B}\right\|_2$ between randomly paired observations for each pair type, where $\sigma_B$ is the coordinate-wise standard deviation estimated from the
held-out pair endpoints. Figure~\ref{fig:block_validation} shows a consistent separation between the two blocks. On PickCube, the normalized response to a viewpoint change is 0.18 in $r_H$ and 1.51 in $r_L$, making the response of the variant block $8.4\times$ larger. Across matched trajectories, the corresponding distances yield a $2.1\times$ difference. The same pattern is more pronounced on PushCube: the viewpoint and the matched-trajectory responses correspond to $14.0\times$ and $3.1\times$ larger changes in $r_L$, respectively. These results provide a representation validation that is distinct from downstream SR. The invariant block remains comparatively stable when viewpoint and execution factors change, whereas the variant block responds more strongly to the same interventions. It is consistent with the intended block structure in which task-level information is retained in $r_H$, while viewpoint and trajectory variation is preferentially expressed through $r_L$.

\begin{table*}[t]
\centering
\caption{
\textbf{Robot manipulation representation ablation results.}
Values report the mean $\pm$ standard deviation over three training seeds.  All variants retain the same partitioned architecture and downstream pipeline as the Full model and remove only the indicated component. 
ID is evaluated using SR@10cm and the task-specific interaction rate
(CONTACT for PushCube and GRASP for PickCube); all OOD columns report
SR@10cm.
}
\vspace{-10pt}
\label{tab:representation_ablation}
\small
\setlength{\tabcolsep}{3.2pt}
\resizebox{\textwidth}{!}{
\begin{tabular}{lccccc|ccccc}
\toprule
& \multicolumn{5}{c|}{\textbf{PushCube}}
& \multicolumn{5}{c}{\textbf{PickCube}} \\
\cmidrule(lr){2-6}
\cmidrule(lr){7-11}
\textbf{Method}
& \multicolumn{2}{c}{ID}
& \multicolumn{3}{c|}{OOD}
& \multicolumn{2}{c}{ID}
& \multicolumn{3}{c}{OOD} \\
\cmidrule(lr){2-3}
\cmidrule(lr){4-6}
\cmidrule(lr){7-8}
\cmidrule(lr){9-11}
& SR@10cm
& CONTACT
& Camera
& Lighting
& Visual
& SR@10cm
& GRASP
& Camera
& Lighting
& Visual \\
\midrule

Full ($r_H$)
& $99.3 \pm 0.6$
& $100.0 \pm 0.0$
& $71.5 \pm 10.4$
& $99.7 \pm 0.6$
& $64.5 \pm 12.7$
& $45.0 \pm 2.4$
& $97.0 \pm 1.4$
& $\mathbf{30.3 \pm 6.1}$
& $42.7 \pm 4.5$
& $28.3 \pm 3.9$ \\

Full ($r_H+r_L$)
& $99.3 \pm 0.6$
& $100.0 \pm 0.0$
& $75.3 \pm 8.1$
& $98.3 \pm 1.5$
& $\mathbf{70.2 \pm 12.9}$
& $\mathbf{46.3 \pm 4.9}$
& $96.7 \pm 0.5$
& $27.7 \pm 0.9$
& $43.3 \pm 2.9$
& $\mathbf{30.3 \pm 1.7}$ \\
\midrule

w/o Predictive ($r_H$)
& $98.7 \pm 0.6$
& $100.0 \pm 0.0$
& $65.7 \pm 6.4$
& $99.7 \pm 0.6$
& $54.0 \pm 14.3$
& $39.7 \pm 4.5$
& $92.7 \pm 1.2$
& $20.7 \pm 0.9$
& $35.7 \pm 3.9$
& $19.0 \pm 0.8$ \\

w/o Predictive ($r_H+r_L$)
& $98.7 \pm 0.6$
& $100.0 \pm 0.0$
& $65.3 \pm 8.8$
& $98.0 \pm 1.0$
& $59.8 \pm 12.4$
& $39.3 \pm 2.5$
& $94.0 \pm 1.4$
& $18.3 \pm 1.7$
& $39.3 \pm 3.4$
& $18.7 \pm 1.2$ \\
\midrule

w/o Invariance ($r_H$)
& $98.7 \pm 0.6$
& $100.0 \pm 0.0$
& $19.7 \pm 9.5$
& $88.7 \pm 6.7$
& $15.7 \pm 11.0$
& $44.3 \pm 1.2$
& $92.0 \pm 4.2$
& $24.0 \pm 3.6$
& $45.3 \pm 1.9$
& $15.3 \pm 0.9$ \\

w/o Invariance ($r_H+r_L$)
& $98.0 \pm 1.0$
& $99.7 \pm 0.6$
& $15.7 \pm 10.1$
& $93.3 \pm 1.5$
& $6.5 \pm 4.5$
& $41.0 \pm 3.7$
& $92.3 \pm 0.9$
& $24.0 \pm 3.3$
& $41.7 \pm 3.4$
& $13.0 \pm 0.8$ \\
\midrule

w/o Diversity ($r_H$)
& $\mathbf{100.0 \pm 0.0}$
& $100.0 \pm 0.0$
& $78.0 \pm 4.8$
& $99.7 \pm 0.6$
& $67.0 \pm 9.5$
& $45.0 \pm 2.8$
& $\mathbf{99.0 \pm 0.0}$
& $21.7 \pm 1.2$
& $44.0 \pm 1.4$
& $18.7 \pm 3.1$ \\

w/o Diversity ($r_H+r_L$)
& $\mathbf{100.0 \pm 0.0}$
& $100.0 \pm 0.0$
& $\mathbf{79.3 \pm 6.8}$
& $99.7 \pm 0.6$
& $67.8 \pm 15.3$
& $43.3 \pm 0.5$
& $98.0 \pm 0.8$
& $21.3 \pm 2.9$
& $\mathbf{47.0 \pm 3.3}$
& $20.3 \pm 2.1$ \\

\bottomrule
\end{tabular}
}
\end{table*}

\paragraph{Controlled Intervention.}
To evaluate the robustness implication of Theorem~\ref{th3}, we compare two representation interfaces: one exposes only the invariant block $r_H$, while the other exposes both $r_H$ and the variant block $r_L$. We hold the environment state and all other policy inputs fixed and intervene only on $r_L$. Across both tasks, the $r_H$-only policy retains approximately $100\%$ of its clean performance, whereas the $r_H+r_L$ policy becomes progressively more sensitive as the magnitude of corruption increases. Table~\ref{tab:intervention_retention} shows that routing predictions through $r_H$ provides a robust interface to perturbations confined to $r_L$, while exposing the variant block introduces an additional pathway through which such perturbations can affect downstream decisions. These results support Theorem~\ref{th3} that once the invariant block is separated from variation, predictions based only on $r_H$ are substantially more robust to perturbations of the variant component.

\paragraph{Temporal Transferability}
\begin{figure}[t]
  \centering
  \includegraphics[width=0.8\linewidth]{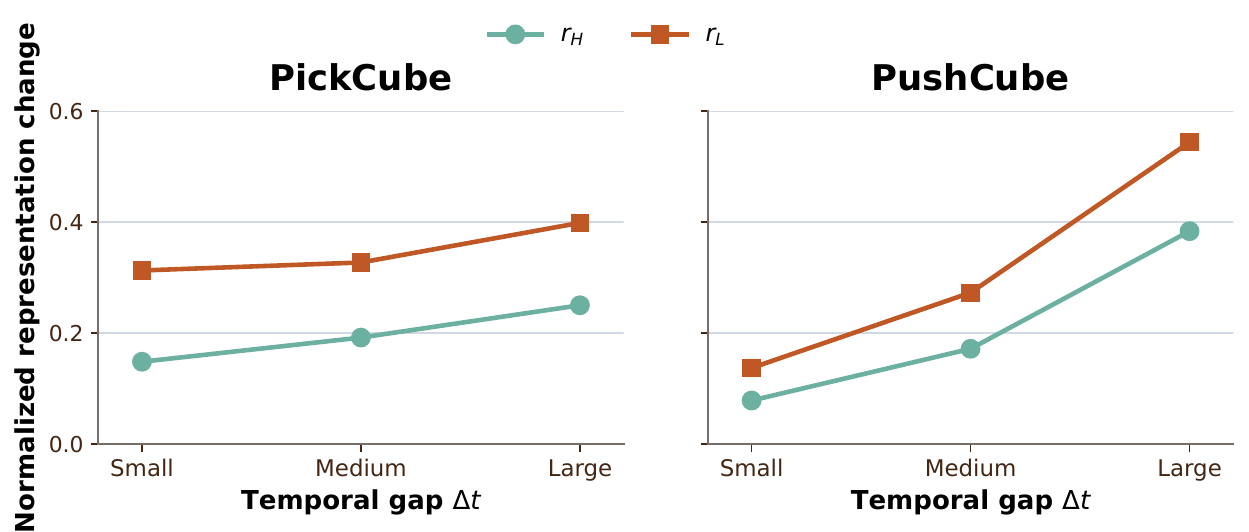}
  \caption{
\textbf{Temporal transferability.} Normalized representation changes between state-matched observations from different trajectories and timesteps averaged over three seeds.}
\label{fig:temporal_transferability}
\vspace{-10pt}
\end{figure}
We evaluate whether $r_H$ transfers more consistently than $r_L$ across execution time using state-matched asynchronous pairs. Each pair is drawn from different validation trajectories in the same $H$-group and at different normalized timesteps. Simulator states are used only to construct the pairs, while the encoder receives observations alone. We group the pairs into small, medium and large normalized timestep gaps and measure each block's representation change relative to its random-pair distance. For each block, we measure its standardized representation distance $D_B$ between the paired observations and normalize it by the mean distance between observations with different $H$: $S_B(v)=\frac{\mathbb{E}_{(x_a,x_b)\sim v}\left[D_B(x_a,x_b)\right]
}{\mathbb{E}_{(x_a,x_b)\sim \mathrm{random}}\left[D_B(x_a,x_b)\right]}$. A smaller $S_B$ indicates that block $B$ transfers more consistently across
the tested temporal variation. We use $S_H/S_L$ as the relative stability
measure, where a value below one indicates greater temporal stability in
$r_H$.All results are averaged over three independently trained seeds without additional training.

Figure~\ref{fig:temporal_transferability} shows a consistent separation between the two representation blocks. Across both tasks and all temporal gaps, $r_H$ changes less than $r_L$ between state-matched. As the temporal gap increases, both representations exhibit larger changes, reflecting the increasing difference between observations collected farther apart in execution, while the representation shift in $r_H$ is always smaller than that in $r_L$. The asynchronous pairs preserve the shared task condition and current physical state while changing the trajectory-specific temporal context. Accordingly, the relative stability of $r_H$ indicates that it captures the information invariant across these observations, whereas the greater variation in $r_L$ reflects its sensitivity to execution-specific factors. Since a downstream predictor operating on $r_H$ encounters a smaller representation shift across timesteps, the learned invariant block provides a more transferable interface for prediction throughout task execution.

\end{document}